\documentclass[letterpaper]{article} 
\usepackage{aaai2027}  
\usepackage[hyphens]{url}  
\usepackage{graphicx} 
\usepackage{natbib}  
\usepackage{caption} 
\usepackage{algorithm}
\usepackage{algorithmic}

\usepackage{newfloat}
\usepackage{listings}
\usepackage{amsmath} 
\usepackage{amssymb}
\usepackage{multirow}
\DeclareCaptionStyle{ruled}{labelfont=normalfont,labelsep=colon,strut=off} 
\floatstyle{ruled}
\newfloat{listing}{tb}{lst}{}
\floatname{listing}{Listing}

\usepackage{amsmath}
\usepackage{amsthm}
\usepackage{algorithm}
\usepackage{algorithmic}
\newtheorem{theorem}{Theorem}[section] 
\newtheorem{lemma}[theorem]{Lemma} 
\newtheorem{proposition}[theorem]{Proposition}
\newtheorem{corollary}[theorem]{Corollary}
\newtheorem{definition}[theorem]{Definition}
\usepackage{newfloat}
\usepackage{listings}
\usepackage{amsmath} 
\usepackage{amssymb}
\usepackage{multirow}

\usepackage{booktabs}

\title{FairDiff: Mitigating the Self-Reinforcing Matthew Effect in \\Diffusion Recommender Models}
\author{
    Song-Li Wu\textsuperscript{\rm 1},
    Xianquan Wang\textsuperscript{\rm 3},
    Zhaocheng Du\corresponding\textsuperscript{\rm 2},
    Weinan Gan\textsuperscript{\rm 2},
    Jingyi Wang\textsuperscript{\rm 1}
}
\affiliations{
    \textsuperscript{\rm 1}Tsinghua University,
    \textsuperscript{\rm 2}Huawei Noah’s Ark Lab,
    \textsuperscript{\rm 3}University of Science and Technology of China
    
}

\begin{document}

\maketitle

\begin{abstract}
While the "Matthew Effect" and filter bubbles are widely recognized outcome-level biases in recommender systems, we reveal that Diffusion Recommender Models (DRMs) uniquely compound this issue through their generative dynamics. Rather than merely inheriting data imbalances, DRMs trigger a self-reinforcing amplification of popularity bias. We identify that this phenomenon is driven by two compounding mechanisms. First, while optimization loss is universally dominated by high-frequency items across recommenders, DRMs suffer from a unique structural prior mismatch during generation. Because the forward terminal distribution of long-tailed data deviates significantly from the standard Gaussian prior, reverse sampling trajectories inherently collapse toward high-density popular items, fundamentally suppressing niche item generation.  To dismantle this self-reinforcing loop, we propose FairDiff, a plug-and-play fairness-aware diffusion framework. To overcome the popularity-dominated loss, we introduce Popularity Condition Guidance (PCG). Rather than altering the training objective, PCG acts as an inference-time distributional reweighting mechanism, mathematically reshaping the score-based gradient field to penalize high-popularity regions and guide trajectories toward niche semantics. Furthermore, we design a Semantic Calibration (SC) Module to bridge the prior mismatch, aligning the forward and reverse distributions via one-step optimal transport. Comprehensive evaluations demonstrate that FairDiff achieves state-of-the-art performance while effectively mitigating the self-reinforcing Matthew Effect, highlighting its value as a general framework for DRMs. 
\end{abstract}


\section{Introduction}

Recommender systems rely on sparse, long-tailed, and noisy user–item interactions, posing substantial challenges to traditional discriminative models~\citep{koren2021advances,wang2026grpo}. Diffusion Recommender Models (DRMs) address these challenges by learning distribution-level representations through a generative paradigm that progressively injects Gaussian noise into observed interactions (forward process) and reconstructs them via a parameterized Markov chain (reverse process). By modeling preference evolution as a stochastic denoising trajectory, DRMs enable robust capture of complex dependencies, promote stable generation, and naturally support the integration of heterogeneous signals. Consequently, DRMs have emerged as a powerful paradigm across a wide range of recommendation scenarios, including next-item generation~\citep{jiang2024diffkg,qin2023diffusion,tomasi2024diffusion}, preference representation learning~\citep{xuan2024diffusion,zhang2024metadiff,zhao2024denoising,wu2026capedit}, and temporal dynamics modeling~\citep{ma2024plug,li2024recdiff,wu2026activemem}.  

While the "Matthew Effect" and filter bubbles are widely recognized as outcome-level biases in general recommender systems, we observe that DRMs do not merely inherit this data-driven popularity bias. Instead, they uniquely compound it through their generative dynamics, leading to what we term a self-reinforcing Matthew Effect. As shown in Table~\ref{tab:performance}, Table~\ref{tab:item_user_performance} and appendix, although DRMs achieve superior overall performance compared to existing recommender models, this improvement is largely driven by substantial gains on popular items. In contrast, for niche items, DRMs consistently underperform. This uneven performance distribution indicates that the apparent overall superiority of DRMs comes at the cost of exacerbating the accuracy disparity between popular and niche items, thereby actively reinforcing the Matthew Effect rather than alleviating it.  

\begin{figure*}[htbp]

\centering

\includegraphics[width=0.8\textwidth]{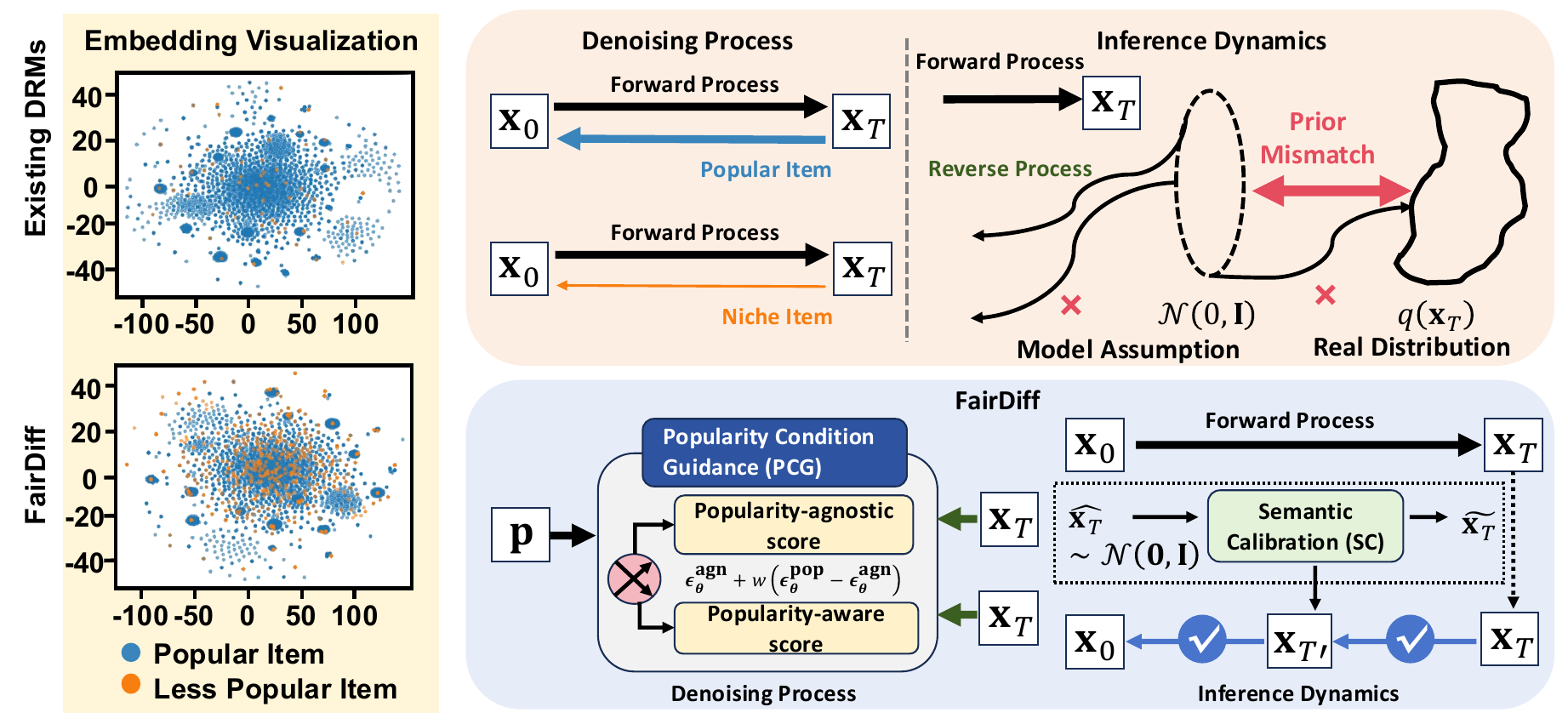}

\caption{Comparison of DRMs and FairDiff.}

\label{tesear}

\end{figure*}

We attribute this self-reinforcing phenomenon in DRMs to two complementary mechanisms that compound across the model's lifecycle, as illustrated in Figure~\ref{tesear}. First, during the diffusion phase, the reconstruction objective is inherently biased under long-tailed distributions. While the optimization loss being dominated by high-frequency items is a universal challenge across most recommendation paradigms, DRMs suffer from a second, unique vulnerability during the generation phase: a structural prior mismatch. Because the forward terminal distribution of long-tailed data deviates significantly from the standard Gaussian prior, reverse sampling trajectories—initiating from standard Gaussian noise—are structurally mismatched. This geometric discrepancy biases the generative path toward high-density modes (popular items) from the very first denoising steps, fundamentally suppressing the diversity required to recommend niche items. Together, these mechanisms form a self-reinforcing loop—encoding bias during diffusion and amplifying it during generation—thereby severely limiting recommendation performance for tail items.  

To address these challenges, we propose FairDiff, a unified, plug-and-play framework consisting of two complementary modules that can be seamlessly integrated into existing diffusion recommender models without altering their architecture or training procedure. First, rather than attempting to alter the universally skewed training objective, we introduce the Popularity Condition Guidance (PCG) module as an inference-time distributional reweighting intervention. PCG mathematically reshapes the score-based gradient field during reverse sampling by leveraging the discrepancy between the model’s outputs under true popularity conditioning and those under a sampled popularity condition. This explicitly penalizes high-popularity regions and redistributes probability mass toward niche items, effectively bypassing the popularity-driven score dominance. Second, to bridge the prior mismatch, we design a Semantic Calibration (SC) module. SC uses a one-step static Optimal Transport (OT) mapping to directly align the reverse initial distribution with the empirical forward terminal distribution. This single-step calibration corrects the structural mismatch, ensuring reverse trajectories are properly initialized to preserve niche semantics.  

Our contributions are summarized as follows:

\begin{itemize}

\item We reveal that Diffusion Recommender Models suffer from a unique, self-reinforcing Matthew Effect, distinctly caused by the composition of optimization bias and a DRM-specific structural prior mismatch.

\item We propose FairDiff, a unified and plug-and-play framework that systematically mitigates popularity bias. Crucially, it bypasses the inherently biased loss landscape via an inference-time distributional reweighting mechanism (PCG) and an optimal transport initialization (SC).

\item We conduct extensive experiments on various representative recommendation tasks, demonstrating that FairDiff achieves state-of-the-art performance and structurally reshapes generation trajectories to favor niche items, without introducing additional computational overhead.

\end{itemize}  

\section{Related Work}

\subsection{Diffusion Recommender Models}
Diffusion models have been widely applied to recommendation tasks, including next-item generation~\citep{tomasi2024diffusion, qin2023diffusion}, user preference modeling~\citep{xuan2024diffusion, wu2026reaction,zhang2024metadiff, zhao2024denoising,wu2026codeskill}, and temporal dynamics modeling~\citep{ma2024plug}. Recent advancements also integrate auxiliary signals like knowledge graphs~\citep{jiang2024diffkg,wu2026id} and multi-modal features~\citep{jiang2024diffmm} to enhance semantic alignment. However, despite their strong generative capabilities, existing Diffusion Recommender Models (DRMs) universally rely on standard Gaussian initializations during reverse sampling. This standard assumption ignores the complex data manifold of recommendation scenarios, leaving current DRMs inherently exposed to the unique structural biases we address in this work.

\subsection{Matthew Effect in Recommendation}
The Matthew Effect remains a fundamental challenge in recommender systems. Existing studies primarily treat this phenomenon as an \emph{externally observable outcome}, mitigating it through exposure control, regularization~\citep{zhao2025hybrid, gao2023alleviating, zheng2021disentangling,wu2026finesid}, or post-hoc diversity enhancements~\citep{anderson2020algorithmic, hansen2021shifting, liang2021enhancing}. Empirical analyses also confirm how item popularity skews outcomes in traditional models~\citep{wang2018quantitative} and real-world platforms~\citep{liu2021examining,han2025controlling}. 
However, little attention has been paid to how modern generative paradigms internally \emph{self-reinforce} this bias. Crucially, while high-frequency items dominating the optimization loss is a universal challenge across all recommendation models, DRMs suffer from an additional, uniquely generative vulnerability. The reliance on stochastic reverse generation introduces a structural prior mismatch between the forward terminal distribution and the reverse initialization, systematically amplifying popularity bias over time. This observation motivates us to analyze how mismatched generative sampling intrinsically reinforces the Matthew Effect in DRMs, and to design inference-time interventions that correct these trajectory distortions without modifying the universally biased loss objective.

\section{Preliminary}
\subsection{Diffusion Models}
Diffusion Models (DMs) transform structured data into noise via a forward process and reverse it to generate samples. In recommendations, item embeddings $z_0$ are obtained by looking up the latent vector of each item from the predefined item embedding matrix $E\in\mathbb{R}^{n\times d}$ for items in each user’s historical interaction sequence extracted from the user-item interaction matrix $R$. These embeddings $z_0$ are then encoded into latent variables $\mathbf{x}_0$ for diffusion. The diffusion model generates $\hat{\mathbf{x}}_0$, which is decoded to $\hat{\mathbf{z}}_0$, a reconstructed embedding for ranking or matching items.

\noindent\textbf{Forward Diffusion Process.}
The forward process gradually corrupts $\mathbf{x}_0$ over $T$ discrete steps by adding Gaussian noise, producing a Markov chain of noisy latent variables $\mathbf{x}_1, \dots, \mathbf{x}_T$. This is defined by the conditional transition:
\begin{equation}
q(\mathbf{x}_t | \mathbf{x}_{t-1}) = \mathcal{N}(\mathbf{x}_t; \sqrt{1-\beta_t}\mathbf{x}_{t-1}, \beta_t \mathbf{I}), 
\end{equation}
where $t = 1, \dots, T$, $\beta_t \in (0,1)$ controls the noise scale at each diffusion step $t$, and $\mathbf{I} \in \mathbb{R}^{d \times d}$ is the identity matrix ensuring isotropic variance. Additionally, for any $t > 0$, the forward process can be expressed in one step from $\mathbf{x}_0$ to $\mathbf{x}_t$ using two distributions:

\begin{equation}
\mathbf{x}_t = \sqrt{\lambda_t}\mathbf{x}_0 + \sqrt{1-\lambda_t}\mathbf{\epsilon},
\end{equation}
where $\lambda_t = \prod_{s=1}^t (1-\beta_s)$ and $\mathbf{\epsilon} \sim \mathcal{N}(0, \mathbf{I})$.

\noindent\textbf{Reverse Denoising Process.}
To recover structured signals from the noise, DMs train a neural network to approximate the reverse transitions. Starting from Gaussian noise $\hat{\mathbf{x}}_T \sim \mathcal{N}(0, \mathbf{I})$, the model iteratively predicts the denoised latent states $\hat{\mathbf{x}}_{t-1}$ given $\hat{\mathbf{x}}_t$, conditioned on auxiliary recommendation-specific context $\mathbf{c}$:
\begin{equation}
p_\theta(\hat{\mathbf{x}}_{t-1} | \hat{\mathbf{x}}_t, \mathbf{c}) = \mathcal{N}(\hat{\mathbf{x}}_{t-1}; \boldsymbol{\mu}_\theta(\hat{\mathbf{x}}_t, t, \mathbf{c}), \Sigma_t \mathbf{I},
\label{inference}
\end{equation}
where the context $\mathbf{c}$ denotes bidirectional sequence information around replaceable positions, which is embedded into latent vectors, encoded via bidirectional Transformer, and fed as diffusion conditions to generate semantically coherent augmentations and $\boldsymbol{\mu}_\theta(\cdot)$ is the learned mean function parameterized by $\theta$, and $\Sigma_t$ is a variance schedule. As a result, the reverse diffusion trajectory is fully governed by the learned score dynamics, making it sensitive to bias during diffusion process.

\noindent\textbf{Optimization.}  
The learning objective of diffusion models is derived from the variational lower bound (VLB) of the negative log-likelihood $-\log p_\theta(\mathbf{x}_0)$. The VLB can be decomposed into a sum of Kullback–Leibler (KL) terms between the true posterior $q(\mathbf{x}_{t-1}|\mathbf{x}_t,\mathbf{x}_0)$ and the learned reverse transition $p_\theta(\mathbf{x}_{t-1}|\mathbf{x}_t)$, together with a reconstruction term for $\mathbf{x}_0$~\citep{ho2020denoising,li2024recdiff}. With the posterior distribution
\begin{equation}
q(\mathbf{x}_{t-1}|\mathbf{x}_t,\mathbf{x}_0) = 
\mathcal{N}\!\left(\mathbf{x}_{t-1}; \tilde{\boldsymbol{\mu}}_t(\mathbf{x}_t,\mathbf{x}_0), \tilde{\beta}_t \mathbf{I}\right),
\end{equation}
and parameterization of the mean by a noise predictor $\boldsymbol{\epsilon}_\theta$. In practice, following~\citep{ho2020denoising,walker2022recommendation,jiangzhou2024dgrm}, the objective is simplified into a noise prediction loss, which avoids explicitly evaluating KL terms:
\begin{equation}
\label{loss}
\mathcal{L}(\theta) = \mathbb{E}_{t,\boldsymbol{\epsilon}}
\left[
\left\| \boldsymbol{\epsilon} - \boldsymbol{\epsilon}_\theta\!\left(\sqrt{\lambda_t}\mathbf{x}_0 + \sqrt{1-\lambda_t}\,\boldsymbol{\epsilon},\, t,\, \mathbf{c}\right) \right\|_2^2
\right],
\end{equation}
where the noise-prediction objective is equivalent to learning the score function $\nabla_{\mathbf{x}_t}\log p_t(\mathbf{x}_t \mid \mathbf{c})$ of the perturbed data distribution. Crucially, this score is learned under the empirical interaction distribution, implicitly entangling semantic preference signals with dataset-induced frequency effects. Intuitively, the model learns to predict the noise injected at step $t$, thereby recovering the clean embedding $\mathbf{x}_0$ from noisy samples.

\noindent\textbf{Inference.}
User and item embedding generation starts by sampling a noise vector $\hat{\mathbf{x}}_T \sim \mathcal{N}(0, \mathbf{I})$ and iteratively denoising it via the learned transitions:$\hat{\mathbf{x}}_T \rightarrow \hat{\mathbf{x}}_{T-1} \rightarrow \cdots \rightarrow \hat{\mathbf{x}}_0$. This produces user or item embeddings that incorporate both uncertainty and contextual preference signals. Finally, with the generated user or item representation $\hat{\mathbf{x}}_0$, we can calculate the inner product between user embeddings and item embeddings in the candidate set, then top-K nearest items are selected as recommendation items.

\noindent\textbf{Limitations.}
Diffusion-based Recommendation Models (DRMs) learn user-item interaction distributions from data, but under long-tailed settings, popular items dominate training, leaving insufficient signal for niche items and amplifying popularity bias, leading to the self-reinforcing Matthew Effect.
This bias stems from two intrinsic mechanisms: (i) the diffusion objective in Eq.~\eqref{loss} learns the score function $\nabla_{\mathbf{x}_t}\log p_t(\mathbf{x}_t \mid \mathbf{c})$ under a highly imbalanced distribution, causing gradients to be dominated by high-frequency items; and (ii) reverse sampling from a standard Gaussian noise $\hat{\mathbf{x}}_T \sim\mathcal{N}(\mathbf{0},\mathbf{I})$ mismatched with $\mathbf{x}_T$ biases early denoising toward popular-item regions.
These effects jointly reinforce popularity dominance, leading to biased and suboptimal recommendations (see appendix).


\section{Method}
FairDiff consists of two key modules: Popularity Condition Guidance (PCG) and Semantic Calibration (SC). PCG provides diffusion-time guidance that introduces explicit control over popularity-sensitive score components, reshaping the reverse diffusion dynamics. Semantic Calibration (SC) applies a one-step static Optimal Transport alignment to bridge the mismatch between the forward terminal distribution and the reverse initialization, ensuring consistency between forward and reverse diffusion. Importantly, these two modules do not modify the underlying DRM architecture, allowing them to be seamlessly used as plug-in components.

\subsection{Popularity Condition Guidance}
\label{sec:pcg}

As discussed in previous section, DRMs tend to amplify popularity bias during diffusion process, particularly under long-tailed item distributions.
A key reason is that the learned score function $\nabla_{\mathbf{x}_t} \log p_t(\mathbf{x}_t \mid \mathbf{c})$ is dominated by high-frequency (popular) items, making it difficult to accurately recover representations of niche items.
To explicitly disentangle popularity effects from semantic preference signals, we propose \emph{Popularity Condition Guidance} (PCG).
Specifically, PCG treats popularity as an \emph{independent and controllable guidance signal} in the reverse diffusion process.

To this end, we extend the denoising network to take an explicit popularity condition $\mathbf{p}$ as input.
Here, $\mathbf{p}$ represents the normalized popularity of items, which is computed from historical interaction frequencies, and is treated as a fixed input during reverse diffusion.
Accordingly, the noise predictor is parameterized as
$\boldsymbol{\epsilon}_\theta(\mathbf{x}_t, t, \mathbf{c}, \mathbf{p})$,
which induces a popularity-conditioned reverse transition:
\begin{equation}
p_\theta(\hat{\mathbf{x}}_{t-1} \mid \hat{\mathbf{x}}_t, \mathbf{c}, \mathbf{p})
=
\mathcal{N}\!\left(
\hat{\mathbf{x}}_{t-1};
\boldsymbol{\mu}_\theta(\hat{\mathbf{x}}_t, t, \mathbf{c}, \mathbf{p}),
\Sigma_t \mathbf{I}
\right),
\end{equation}
where $\boldsymbol{\mu}_\theta(\cdot)$ follows the standard noise-parameterization used in diffusion models.
Under this formulation, $\boldsymbol{\epsilon}_\theta(\mathbf{x}_t, t, \mathbf{c}, \mathbf{p})$ is proportional to the score function
$\nabla_{\mathbf{x}_t} \log p_\theta(\mathbf{x}_t \mid \mathbf{c}, \mathbf{p})$.

To isolate popularity-induced components in the score, we introduce an auxiliary popularity condition
$\hat{\mathbf{p}} \sim q(\mathbf{p})$, sampled from the marginal popularity prior and independent of $\mathbf{x}_t$.
Conditioning on $\hat{\mathbf{p}}$ yields a popularity-agnostic estimate of the score:
\begin{equation}
\boldsymbol{\epsilon}_\theta^{\text{agn}}(\mathbf{x}_t, t, \mathbf{c})
=
\mathbb{E}_{\hat{\mathbf{p}} \sim q(\mathbf{p})}
\left[
\boldsymbol{\epsilon}_\theta(\mathbf{x}_t, t, \mathbf{c}, \hat{\mathbf{p}})
\right],
\end{equation}
which captures semantic preference signals while marginalizing out item popularity.
In practice, this expectation is efficiently approximated using a single Monte Carlo sample.

Given the popularity-aware noise prediction
$\boldsymbol{\epsilon}_\theta^{\text{pop}}(\mathbf{x}_t, t, \mathbf{c})
=
\boldsymbol{\epsilon}_\theta(\mathbf{x}_t, t, \mathbf{c}, \mathbf{p})$,
PCG constructs the guided noise predictor by linearly interpolating between the popularity-agnostic and popularity-aware estimates:
\begin{align}
\hat{\boldsymbol{\epsilon}}_\theta(\mathbf{x}_t, t, \mathbf{c}, \mathbf{p})&=
\boldsymbol{\epsilon}_\theta^{\text{agn}}(\mathbf{x}_t, t, \mathbf{c})
\notag\\ 
&+
w_{\text{PCG}}
\Big(
\boldsymbol{\epsilon}_\theta^{\text{pop}}(\mathbf{x}_t, t, \mathbf{c})
-
\boldsymbol{\epsilon}_\theta^{\text{agn}}(\mathbf{x}_t, t, \mathbf{c})
\Big),
\label{eq:pcg}
\end{align}
where $w_{\text{PCG}} \in [0,1]$ controls the strength of popularity guidance.
When $w_{\text{PCG}} = 0$, the reverse diffusion is fully popularity-agnostic, while $w_{\text{PCG}} = 1$ recovers the standard popularity-conditioned denoising process. 
The guided noise predictor $\hat{\boldsymbol{\epsilon}}_\theta$ is then used to compute the reverse mean
$\boldsymbol{\mu}_\theta^{\text{PCG}}(\mathbf{x}_t, t, \mathbf{c}, \mathbf{p})$ and perform sampling, following the same update rule as in Eq.~\eqref{inference}.
Importantly, PCG introduces no additional model parameters or training objectives and can be seamlessly applied at inference time. From a theoretical standpoint, PCG can be viewed as a process-level correction of optimization-induced popularity bias, achieved by decomposing the learned score into popularity-agnostic and popularity-sensitive components. By explicitly regulating the contribution of popularity within the score dynamics, PCG reshapes the reverse diffusion trajectories. The theoretical foundations of this perspective are established in appendix.

\subsection{Semantic Calibration Module}
\label{sec:sc}

While the PCG module addresses the optimization bias in gradient estimation, DRMs face another critical challenge at the inference stage: the \emph{Structural Prior Mismatch}.
In this subsection, we identify the root cause of this mismatch and propose \textbf{Semantic Calibration (SC)}, a lightweight optimal transport strategy to bridge this gap with one step calibration.

\subsubsection{Prior Mismatch}

DRMs assume that the forward process transforms embedding $\mathbf{x}_0$ into pure isotropic Gaussian noise $\mathcal{N}(\mathbf{0}, \mathbf{I})$ at the terminal timestep $T$. Consequently, the reverse generation process is initialized by sampling $\hat{\mathbf{x}}_T \sim \mathcal{N}(\mathbf{0}, \mathbf{I})$.
However, in the context of recommendation, embeddings lie on a highly sparse and complex manifold.
Practically, the distribution of data diffused to time $T$, denoted as $q(\mathbf{x}_T)$, often deviates significantly from the standard Gaussian prior:
\begin{equation}
    q(\mathbf{x}_T) = \int q(\mathbf{x}_T | \mathbf{x}_0) p_{\text{data}}(\mathbf{x}_0) d\mathbf{x}_0 \neq \mathcal{N}(\mathbf{0}, \mathbf{I}).
\end{equation}
We term this discrepancy the \emph{prior mismatch}.
Because the reverse denoising process is trained to map from $q(\mathbf{x}_T)$ back to $\mathbf{x}_0$, initializing generation with $\hat{\mathbf{x}}_T$ (which comes from a different distribution) introduces a systematic trajectory error.
In RecSys, this error manifests as a tendency to collapse towards the "mean" of the dataset during the early denoising steps, causing the model to recommend generic, popular items rather than personalized niche items.
To analyze this formally, existing research prove that the reverse diffusion can be modeled as a Probability Flow ODE \citep{song2020score},
The error in the initial distribution $p_T$ propagates through the ODE trajectory, creating a divergence between the generated distribution and the true data distribution.
The upper bound of this divergence accumulates over time, implying that the early steps of generation are wasted on correcting the initialization rather than refining user preferences. We present the corresponding theorems and  proofs in appendix.

\subsubsection{Semantic Calibration}
Instead of relying on the computationally expensive method of correcting the entire ODE trajectory, we focus on correcting the \emph{initialization}.
We aim to find a mapping function $\mathcal{T}$ that transports the standard Gaussian noise to the empirical terminal distribution of the embedding.
We formulate this as a static Optimal Transport (OT) problem minimizing the transport cost:
\begin{equation}
    \min_{\mathcal{T}} \mathbb{E}_{\hat{\mathbf{x}_T} \sim \mathcal{N}(\mathbf{0}, \mathbf{I})} \left[ \| \mathcal{T}(\hat{\mathbf{x}_T}) - \text{stop\_grad}(\mathbf{x}_T) \|_2^2 \right],
\end{equation}
where $\mathbf{x}_T$ are samples obtained by running the forward diffusion process on real user data.
In practice, $\mathcal{T}$ can be parameterized as a lightweight projection network (e.g., a multi-layer perceptron).
During inference, rather than starting from random noise, we introduce a \emph{calibrated initialization}:
\begin{equation}
    \tilde{\mathbf{x}}_T = \mathcal{T}(\hat{\mathbf{x}}_T), \quad \text{where } \hat{\mathbf{x}}_T \sim \mathcal{N}(\mathbf{0}, \mathbf{I}).
\end{equation}
This $\tilde{\mathbf{x}}_T$ ensures the reverse process starts on the correct trajectory.

\subsubsection{Efficient Truncated Sampling}

A key advantage of Semantic Calibration is that it enables \emph{Truncated Diffusion}.
Since $\tilde{\mathbf{x}}_T$ is already aligned with the data distribution, the chaotic early steps of the reverse process (typically $T \to T/2$) become redundant.
We can therefore bypass the initial denoising phase and start the generation from a later timestamp $T' < T$.
This design significantly reduces inference latency while mitigating the popularity bias caused by incorrect initialization.

\begin{table*}[htbp]
\centering
\small
\caption{Performance comparison across datasets (metrics: HR@K (H@K) and NDCG@K (N@K), K $\in {5,10,20}$).   The t-test results show that our performance advantage over the previous SOTA method is statistically significant. The improvement is statistically significant with $p < 10^{-2}$ ($\star$: $p < 10^{-2}$, $\star\star$: $p < 10^{-4}$).}
 \resizebox{\textwidth}{!}{%
\begin{tabular}{c|c|cccccccc|cc|cc|cc}
\toprule[1.5pt]
\textbf{Dataset} & \textbf{Metric} & \textbf{SASRec}  &  \textbf{BERT4Rec} &  \textbf{TiMiRec} & \textbf{TIGER} & \textbf{BASRec} & \textbf{HSTU} & \textbf{ACVAE}& \textbf{CSRec}& \textbf{DreamRec} & \textbf{+FairDiff} &  \textbf{DiffuRec} & \textbf{+FairDiff}  & \textbf{CDiff4Rec} & \textbf{+FairDiff}  \\
\midrule
\multirow{6}{*}{\textbf{Beauty} }
& H@5 
&3.27 &2.13  &1.90 &3.53 & 5.37& 5.25&2.47 &5.48 & 5.61 & \textbf{6.55\textsuperscript{$\star$$\star$}}  & 5.73 & \textbf{6.94\textsuperscript{$\star$$\star$}} &  5.83 & \textbf{7.29\textsuperscript{$\star$$\star$}}  \\& H@10 
& 6.26    &3.72 &3.34 &6.23 &7.03 &6.29 &3.88&7.08 &7.13 & \textbf{7.99\textsuperscript{$\star$$\star$}}  & 7.21 & \textbf{7.96\textsuperscript{$\star$$\star$}} &  7.71 & \textbf{8.69\textsuperscript{$\star$$\star$}}  \\& H@20
& 8.98   & 5.79 &5.17 &9.63 &10.32 &10.08 &6.12 & 7.53&10.62 & \textbf{11.82\textsuperscript{$\star$$\star$}}  & 10.53 & \textbf{11.76\textsuperscript{$\star$$\star$}} &  10.72 & \textbf{11.90\textsuperscript{$\star$$\star$}}  \\&N@5 
& 2.40    &1.32  &1.24 &2.59 &3.29 &3.04 &1.71 &3.28& 3.49 & \textbf{3.92\textsuperscript{$\star$$\star$}}  & 3.64 & \textbf{4.13\textsuperscript{$\star$$\star$}}  & 3.90 & \textbf{4.39\textsuperscript{$\star$$\star$}}  \\&N@10
& 3.23   &1.83  &1.71 &3.26 &4.21 &3.97 &2.07 & 4.18 &4.19 & \textbf{4.78\textsuperscript{$\star$$\star$}}  & 4.44 & \textbf{4.89\textsuperscript{$\star$$\star$}} &  4.65 & \textbf{5.21\textsuperscript{$\star$$\star$}}  \\& N@20
& 3.66   & 2.35 &2.17 &3.71 &4.95 &4.82 &2.63& 5.08 &4.99 & \textbf{5.85\textsuperscript{$\star$$\star$}} & 5.32 & \textbf{5.93\textsuperscript{$\star$$\star$}} &  5.41 & \textbf{6.01\textsuperscript{$\star$$\star$}}  \\
\midrule
\multirow{6}{*}{\textbf{Toys}} & H@5 
& 4.53    & 1.91 &1.14 &4.24 &5.41 &5.21 &2.24& 5.32&5.47 & \textbf{6.04\textsuperscript{$\star$$\star$}}  & 5.58 & \textbf{6.15\textsuperscript{$\star$$\star$}} &  6.04 & \textbf{6.46\textsuperscript{$\star$$\star$}}\\& H@10
& 6.55   &2.94  &1.75  &6.83 &7.13&6.73 &3.05& 7.16& 7.25 & \textbf{8.18\textsuperscript{$\star$$\star$}}  & 7.22 & \textbf{7.98\textsuperscript{$\star$$\star$}} &  7.70 & \textbf{8.47\textsuperscript{$\star$$\star$}} \\& H@20
& 9.23  &4.59  &2.72 &9.57 &9.51 &8.02 &4.41&9.10 &9.68 & \textbf{10.38\textsuperscript{$\star$$\star$}}  & 9.84 & \textbf{10.62\textsuperscript{$\star$$\star$}} &  10.16 & \textbf{10.87\textsuperscript{$\star$$\star$}} \\&N@5 
& 3.01   &1.16  &0.71 &3.15 &3.11 &3.34 &1.56& 3.96 &4.04 & \textbf{4.62\textsuperscript{$\star$$\star$}}  & 4.18 & \textbf{4.64\textsuperscript{$\star$$\star$}} &  4.47 & \textbf{4.86\textsuperscript{$\star$$\star$}} \\& N@10
& 3.75  &1.49  &0.91 &3.82 &3.92 &4.27 &1.85 & 4.54&4.62 & \textbf{5.27\textsuperscript{$\star$$\star$}}  & 4.75 & \textbf{5.31\textsuperscript{$\star$$\star$}}  & 5.00 & \textbf{5.49\textsuperscript{$\star$$\star$}}  \\& N@20
& 4.33  &1.90  &1.14 &4.32 &4.71 &4.94 &2.18 &5.18 &5.22 & \textbf{5.97\textsuperscript{$\star$$\star$}} & 5.35 & \textbf{6.03\textsuperscript{$\star$$\star$}}  & 5.65 & \textbf{6.10\textsuperscript{$\star$$\star$}} \\
\midrule
\multirow{6}{*}{\textbf{Steam}} & H@5 
& 4.74   &4.74  &5.02 &4.57 &5.64 &5.62 &5.58 &5.78 &5.96 & \textbf{7.05\textsuperscript{$\star$$\star$}}  & 6.72 & \textbf{7.19\textsuperscript{$\star$$\star$}} &  6.90 & \textbf{7.33\textsuperscript{$\star$$\star$}}\\& H@10
&8.38  &7.94  &9.62& 8.64& 9.28& 9.45& 9.28 & 9.57&9.68 & \textbf{11.10\textsuperscript{$\star$$\star$}}  & 10.51 & \textbf{11.45\textsuperscript{$\star$$\star$}} &  10.09 & \textbf{11.85\textsuperscript{$\star$$\star$}}  \\& H@20
& 13.61   &12.73 &14.89 &14.13 &15.03 &15.06 &14.48&14.89 &15.08 & \textbf{16.93\textsuperscript{$\star$$\star$}}  & 16.09 & \textbf{17.23\textsuperscript{$\star$$\star$}} & 16.89 & \textbf{18.64\textsuperscript{$\star$$\star$}} \\&N@5 
& 2.88   &2.97 &3.87 &2.93 &3.36 &3.48 &3.54&3.62 &3.84 & \textbf{4.49\textsuperscript{$\star$$\star$}} &  4.19 & \textbf{4.77\textsuperscript{$\star$$\star$}} &  4.48 & \textbf{5.03\textsuperscript{$\star$$\star$}} \\&N@10
& 4.05  &4.00  &5.04 &4.09 &4.68 &4.71 &4.73& 4.96 &5.17 & \textbf{5.62\textsuperscript{$\star$$\star$}} &  5.50 & \textbf{5.89\textsuperscript{$\star$$\star$}} &  5.79 & \textbf{6.32\textsuperscript{$\star$$\star$}} \\& N@20
& 5.36   &5.20  &6.36 &5.54 &6.14 &6.12 &6.04&6.27 &6.39 & \textbf{7.02\textsuperscript{$\star$$\star$}} &  7.11 & \textbf{7.44\textsuperscript{$\star$$\star$}} &  7.28 & \textbf{7.86\textsuperscript{$\star$$\star$}}  \\
\toprule[1.5pt]
\end{tabular}
}
\label{tab:performance}
\end{table*}

\section{Experiments}
\label{exp}


We conduct comprehensive experiments on the sequential recommendation task to answer the following key questions: 

\begin{itemize}
\item \textbf{RQ1}: How does FairDiff compare with state-of-the-art diffusion recommender models?
\item \textbf{RQ2}: To what extent does FairDiff attenuate the self-reinforcing Matthew effect in RS?
\item \textbf{RQ3}: What performance gains does FairDiff yield over existing baselines for varying-popularity items and users?
\item \textbf{RQ4}: What is the individual contribution of each module to FairDiff's effectiveness?
\item \textbf{RQ5}: What is the impact of calibrating forward final distribution and reverse initial distribution?
\end{itemize}

\vspace{-3mm}

\subsection{Experimental Settings}

\noindent\textbf{Evaluation Metrics.} We adopt HR\@K and NDCG\@K (K = 5, 10, 20) to evaluate recommendation utility; for brevity, they are denoted as H@K and N@K. To assess the self-reinforcing Matthew effect, we further employ fairness-related metrics from both the consumer and provider perspectives. Consumer-side fairness is measured by $\Delta$Recall and $\Delta$NDCG on the top-20 recommendation lists~\citep{boratto2022consumer}. Provider-side fairness and diversity are evaluated using $\Delta$Exp and APLT~\citep{karimi2023provider}. All metrics are reported as percentages. Each experiment is repeated five times with different random seeds, and the reported results are averaged over all runs.

\noindent\textbf{Datasets.} We evaluate on seven real-world datasets. For sequential recommendation, we use Amazon Beauty and Amazon Toys~\citep{lin2022dual}, MovieLens‑1M~\citep{harper2015movielens}, and Steam~\citep{kang2018self}. For multimodal recommendation, we employ TikTok, Baby, and Sports~\citep{zhou2023bootstrap,xu2024slmrec}. The details of
these datasets are provided in appendix. 


\subsection{Performance Comparison(RQ1)}

To evaluate the effectiveness of FairDiff, we report its performance against a wide range of baselines across Beauty, Toys, and Steam datasets, with detailed results summarized in Table~\ref{tab:performance}. Overall, DRMs consistently outperform traditional discriminative models across all datasets. This superiority arises from their ability to capture the complex user–item interaction distribution, rather than relying on local ranking optimization through point-wise or pair-wise objectives. In contrast, discriminative approaches are more vulnerable to popularity-skewed supervision signals and may suffer from representational bottlenecks in recommendation scenarios. 
Nevertheless, existing DRMs remain susceptible to the self-reinforcing Matthew effect, wherein popularity bias is progressively amplified during the diffusion and generative processes. FairDiff explicitly intervenes in this two processes through the proposed PCG and SC modules, serving as a plug-in enhancement that preserves the expressive capacity of DRMs while mitigating bias accumulation and structural drift. 
Consequently, FairDiff achieves consistent and statistically significant improvements across all datasets and diffusion backbones, with particularly pronounced gains in NDCG, indicating superior ranking quality at top positions rather than mere improvements in hit rate.

\begin{table}[htbp]
  \centering
  \small
  \caption{
Evaluation of the self-reinforcing Matthew effect.
DRMs achieve strong utility (Recall@10 and NDCG@10), 
but exhibit substantial popularity amplification, reflected by large $\Delta$Recall, 
$\Delta$NDCG, and $\Delta$Exposure as well as limited long-tail coverage (APLT). 
By contrast, FairDiff significantly mitigates popularity bias while simultaneously improving utility, 
achieving a superior fairness–utility trade-off. 
All metrics are normalized to [0,1]; higher values indicate better performance 
(except for $\Delta$ metrics, where lower is better).
}
  \label{tab:ml1m-results-transposed}
  \setlength{\tabcolsep}{1.5pt} 
  \renewcommand{\arraystretch}{1.2} 
  \resizebox{0.47\textwidth}{!}{%
  \begin{tabular}{l|c|c|c|c|c|c}
    \toprule
    \textbf{Model}   & \textbf{Recall$@$10 ($\uparrow$)} & \textbf{NDCG$@$10($\uparrow$)} & \textbf{$\Delta$Recall ($\downarrow$)} & \textbf{$\Delta$NDCG ($\downarrow$)} & \textbf{APLT ($\uparrow$)} & \textbf{$\Delta$Exp ($\downarrow$)} \\
       \hline
    
       \textbf{SASRec}      & 9.62 & 10.94 & 0.33 & 1.57 & 9.47 & 86.29 \\
       \textbf{BERT4Rec}& 10.14 & 11.67 & 0.32 & 1.54 & 10.75 & 87.36 \\
       \textbf{TiMiRec}     & 10.34 & 12.82 & 0.16 & 1.37 & 13.73 & 84.98 \\
       \textbf{TIGER}       & 10.46 & 12.94 & 0.14 & 1.34 & 10.50 & 88.78 \\
       \textbf{BASRec}         & 6.45  & 12.76  & 0.35 & 1.66 & 9.49  & 90.54 \\
       \textbf{HSTU}      & 10.35  & 12.04 & 0.48 & 1.43 & 9.63  & 87.02 \\
       \textbf{ACVAE}       & 10.68  & 12.86 & 0.23 & 1.07 & 10.38  & 84.10 \\
       \textbf{CSRec} &10.71 &13.19 &0.31 &1.46 &9.68& 85.82\\
       \hline
          \textbf{DiffuRec}          & 10.95 & 13.73 & 0.62 & 2.33 & 9.04 & 92.26 \\
       \textbf{+FairDiff}    & \textbf{12.01}\textsuperscript{$\star$$\star$} & \textbf{15.28}\textsuperscript{$\star$$\star$} & \textbf{0.12}\textsuperscript{$\star$$\star$} & \textbf{1.18}\textsuperscript{$\star$$\star$} & \textbf{13.81}\textsuperscript{$\star$$\star$} & \textbf{79.81}\textsuperscript{$\star$$\star$} \\\cline{1-7}
          \textbf{DreamRec}        & 11.08 & 13.91 & 1.15 & 2.14 & 9.15 & 95.25 \\  
       \textbf{+FairDiff}    & \textbf{13.02}\textsuperscript{$\star$$\star$} & \textbf{16.17}\textsuperscript{$\star$$\star$} & \textbf{0.18}\textsuperscript{$\star$$\star$} & \textbf{1.15}\textsuperscript{$\star$$\star$} & \textbf{14.73}\textsuperscript{$\star$$\star$} & \textbf{71.87}\textsuperscript{$\star$$\star$} \\ \cline{1-7} 
       \textbf{CDiff4Rec}            & 11.29 & 14.13 & 0.58 & 2.49 & 9.33 & 91.23 \\ 
       \textbf{+FairDiff}     & \textbf{13.85}\textsuperscript{$\star$$\star$} & \textbf{16.25}\textsuperscript{$\star$$\star$} & \textbf{0.09}\textsuperscript{$\star$$\star$} & \textbf{1.18}\textsuperscript{$\star$$\star$} & \textbf{15.53}\textsuperscript{$\star$$\star$} & \textbf{74.82}\textsuperscript{$\star$$\star$} \\ 
    \hline
  \end{tabular}
  }
\end{table}

\subsection{Mitigating the Self-Reinforcing Matthew Effect (RQ2)}

The self-reinforcing Matthew effect in recommendation originates from a cumulative advantage mechanism: items that receive higher exposure are more likely to accumulate interactions, which in turn further amplifies their future exposure probability. This feedback loop gradually distorts the exposure distribution, reinforcing popularity bias and suppressing long-tail providers.
Although explicitly modeling long-term feedback dynamics is beyond the scope of this work, we quantify the structural severity of this effect using provider-side fairness metrics. In particular, exposure disparity ($\Delta$Exp) measures inequality in item exposure allocation, while APLT evaluates the average proportion of long-tail items in recommendation lists. A lower $\Delta$Exp and a higher APLT indicate a weaker cumulative advantage structure and thus a mitigated Matthew effect.
As shown in Table~\ref{tab:ml1m-results-transposed}, existing DRMs still exhibit substantial exposure concentration, even when achieving strong utility performance. This is because diffusion models, despite their generative flexibility, are trained to match the empirical data distribution, which itself is long-tailed and popularity-skewed. Consequently, without explicit structural intervention, the reverse diffusion trajectory tends to preserve—and in some cases amplify—exposure imbalance.
In contrast, FairDiff consistently reduces $\Delta$Exp while simultaneously increasing APLT, without compromising utility metrics. This improvement stems from two complementary mechanisms. PCG explicitly regulates the contribution of popularity signals during generation, preventing excessive amplification of high-frequency items, while SC corrects structural prior mismatch that otherwise biases early denoising steps toward dominant modes. Together, these modules reshape the generative dynamics rather than applying post-hoc reweighting, leading to structurally balanced exposure as an intrinsic property of the generation process.
These results demonstrate that FairDiff mitigates the self-reinforcing Matthew effect at the process level, achieving fairness not through accuracy–fairness trade-offs, but through controlled diffusion dynamics that preserve semantic fidelity while redistributing exposure more equitably.
 
\subsection{Performance across items and users of varying popularity (RQ3)}
Table~\ref{tab:item_user_performance} evaluates performance across cold (bottom 20\%), mid (middle 60\%), and hot (top 20\%) segments for both items and users on toys. FairDiff consistently improves all diffusion backbones, with the most pronounced gains in the cold segments. This validates that semantic calibration mitigates prior mismatch at initialization, preventing early diffusion trajectories from collapsing toward high-density modes and thereby preserving weak preference signals in sparse regions.
Importantly, improvements also remain substantial in the hot segments. Rather than suppressing popularity signals, PCG disentangles popularity-driven gradients from semantic components, enabling strong preference structures to be sharpened instead of diluted. As a result, FairDiff enhances ranking quality in dense regimes without sacrificing expressiveness.
Gains in the mid-frequency range are comparatively moderate, which is expected since vanilla DRMs already handle moderate sparsity reasonably well. In this regime, FairDiff acts as a refinement mechanism rather than a corrective one.
Overall, these results indicate that FairDiff structurally reshapes diffusion trajectories: it amplifies weak signals in sparse regions while preserving discriminative strength in dense regions, leading to consistent improvements across the interaction spectrum.

\begin{table*}[ht]
\centering
\small
\caption{Performance across item popularity and user sequence length. Dim and Pop denotes dimension and popularity.}
\setlength{\tabcolsep}{3pt} 
\renewcommand{\arraystretch}{1.2} 
\resizebox{0.98\textwidth}{!}{
\begin{tabular}{c|c|c|cccccccc|cc|cc|cc} 
\toprule[1.5pt]
\multirow{1}{*}{\textbf{Metric}} & \multirow{1}{*}{\textbf{Dim}} &  \multirow{1}{*}{\textbf{Pop}}&
\textbf{SASRec}   & \textbf{BERT4Rec} &  \textbf{TiMiRec} & \textbf{TIGER} & \textbf{BASRec} & \textbf{HSTU} & \textbf{ACVAE}& \textbf{CSRec}
& \textbf{DiffuRec} & \textbf{+FairDiff} & \textbf{DreamRec} & \textbf{+FairDiff} & \textbf{CDiff4Rec} & \textbf{+FairDiff} \\ 
\midrule 
\multirow{6}{*}{H@20} 
& \multirow{3}{*}{Item} & Cold& 2.85 & 2.12 & 3.67 & 3.92 & 5.12 & 5.18 & 2.56 & 5.32   &2.51 & \textbf{5.59} & 2.74 & \textbf{5.93} & 2.87 & \textbf{6.45} \\
         &             & Mid& 7.12  & 4.37   & 7.45 & 8.12 & 7.98 & 4.12 & 7.55 & 7.41    &7.87 & \textbf{10.07} & 8.02 & \textbf{10.68} & 10.54 & \textbf{12.73} \\
         &            & Hot & 9.42  & 10.23  & 9.05 & 12.75 & 12.08 & 13.90 & 12.32  & 12.94 &  16.09& \textbf{17.35}& 17.21& \textbf{18.58}& 18.08& \textbf{19.14} \\\cmidrule(lr){2-17}
 &\multirow{3}{*}{User} & Cold & 5.72  & 7.32 &  8.05 & 9.27 & 9.12 & 7.15 & 9.44& 9.48  &6.14 &\textbf{11.75}& 7.14& \textbf{14.78}& 7.56 & \textbf{16.28} \\
  &                     & Mid& 8.96  & 8.05  & 9.37 & 10.48 & 10.58 & 10.32 & 8.06 & 10.43  &8.83& \textbf{11.24}& 9.36& \textbf{14.62}& 9.85& \textbf{14.83} \\
   &                    & Hot &9.88  & 9.54 &  11.05 & 11.38 & 11.63 & 12.28 & 9.82&12.41 &12.36& \textbf{13.12}& 13.27& \textbf{14.07}& 14.13& \textbf{15.28} \\
\cmidrule(lr){1-17}
 \multirow{6}{*}{N@20} 
  & \multirow{3}{*}{Item} & Cold &  2.27  & 1.65 &  1.55 & 2.56 & 2.85 & 2.78 & 1.72 & 2.98  &1.74 & \textbf{3.18} & 1.82 & \textbf{3.44} & 1.92 & \textbf{3.74} \\
                 &      & Mid &4.28  & 3.14  & 3.07 & 4.35 & 4.80 & 4.72 & 3.08 & 4.88 &4.55 & \textbf{5.83} & 4.89 & \textbf{6.32} & 5.12 & \textbf{6.62} \\
             &        & Hot & 5.95 & 6.72 & 6.47 & 8.38 & 8.73 & 8.54 & 6.12 & 8.95 &9.34 & \textbf{10.04}& 9.53 & \textbf{10.80} & 9.84 & \textbf{11.04} \\\cmidrule(lr){2-17}
 & \multirow{3}{*}{User} & Cold& 4.64  & 4.88 &  5.17 & 5.68 & 5.92 & 6.02 & 5.41 & 6.14 &4.96 & \textbf{6.58} & 5.02 & \textbf{7.32} & 5.18 & \textbf{7.84} \\
       &                 & Mid &5.04  & 5.12  & 5.75 & 6.18 & 6.42 & 6.37 & 6.13& 6.54 &6.17 & \textbf{7.25} & 6.41 & \textbf{7.51} & 6.39 & \textbf{7.75} \\
       &                & Hot &  5.72  & 5.87 & 6.88 & 7.12 & 7.35 & 7.28 & 5.97 & 7.43 &7.68 & \textbf{8.13} & 7.98 & \textbf{8.43} & 8.34 & \textbf{9.13} \\
\toprule[1.5pt]
\end{tabular}%
}
\label{tab:item_user_performance}
\end{table*}


\begin{table}[tb]
\centering
\small
\caption{Ablation study of FairDiff. PC denotes Popularity Conditioning, TS denotes Truncated Sampling, and INI denotes initialization.}
\label{tab:ablation}
\resizebox{0.48\textwidth}{!}{%
\begin{tabular}{lcccccc}
\toprule
\textbf{Variant} 
& \textbf{Recall$@$10} 
& \textbf{NDCG$@$10} 
& \textbf{$\Delta$Recall} 
& \textbf{$\Delta$NDCG} 
& \textbf{APLT } 
& \textbf{$\Delta$Exp} \\
\midrule
CDiff4Rec 
& 11.29 & 14.13 & 0.58 & 2.49 & 9.33 & 91.23 \\

+ PC ($w{=}1.0$) 
& 12.01 & 14.72 & 0.46 & 2.12 & 10.84 & 86.41 \\

+ PCG ($w{=}0.5$) 
& 12.94 & 15.53 & 0.22 & 1.63 & 13.41 & 80.36 \\

+ PCG ($w{=}0.0$) 
& 12.48 & 15.02 & 0.11 & 1.29 & 14.02 & 78.91 \\

+ SC (OT INI) 
& 12.36 & 15.01 & 0.34 & 1.88 & 12.67 & 84.75 \\

+ SC + TS
& 13.02 & 15.84 & 0.26 & 1.54 & 14.61 & 79.83 \\

\midrule
\textbf{+ FairDiff} 
& \textbf{13.85}$^{\star\star}$ 
& \textbf{16.25}$^{\star\star}$ 
& \textbf{0.09}$^{\star\star}$ 
& \textbf{1.18}$^{\star\star}$ 
& \textbf{15.53}$^{\star\star}$ 
& \textbf{74.82}$^{\star\star}$ \\
\bottomrule
\end{tabular}
}
\end{table}

\subsection{Ablation Study(RQ4)}
\label{sec:ablation}

To better understand the contribution of each component in FairDiff, we conduct comprehensive ablation experiments on the CDiff4Rec backbone. We evaluate both recommendation utility (Recall@10 and NDCG@10) and fairness-related metrics, including consumer-side disparity ($\Delta$Recall and $\Delta$NDCG), provider-side exposure bias ($\Delta$Exp), and long-tail coverage (APLT). All variants are trained under identical settings to ensure fair comparison.

From Table~\ref{tab:ablation}, we observe that introducing popularity conditioning alone improves utility but only moderately alleviates disparity, indicating that simply injecting popularity information does not fundamentally resolve bias. When PCG is applied with a moderate guidance strength ($w{=}0.5$), both consumer-side fairness and long-tail coverage improve substantially while maintaining strong recommendation accuracy. In contrast, the fully popularity-agnostic setting ($w{=}0.0$) further reduces disparity but slightly sacrifices utility, suggesting that completely removing popularity influence may weaken preference signals. These results validate that PCG effectively disentangles semantic preference from popularity effects and enables controllable trade-offs between fairness and accuracy.

Applying Semantic Calibration (SC) alone also leads to consistent improvements across metrics. By aligning the reverse initialization with the empirical terminal distribution, SC mitigates early-stage trajectory distortion and reduces exposure bias. Moreover, when combined with truncated sampling, SC allows the model to skip redundant early denoising steps without degrading performance, indicating that correcting structural prior mismatch enhances both efficiency and recommendation quality.

Most importantly, combining PCG and SC yields the best overall performance. The full FairDiff model achieves the highest Recall@10 and NDCG@10 while simultaneously attaining the lowest disparity and exposure bias and the strongest long-tail coverage. This demonstrates that PCG and SC address distinct yet complementary sources of bias: the former corrects optimization-induced popularity dominance in score estimation, while the latter mitigates inference-time prior mismatch. Their integration provides a principled and effective solution to the self-reinforcing Matthew effect in diffusion-based recommendation.

\begin{figure}[htbp]
    \centering
    \includegraphics[width=0.45\textwidth]{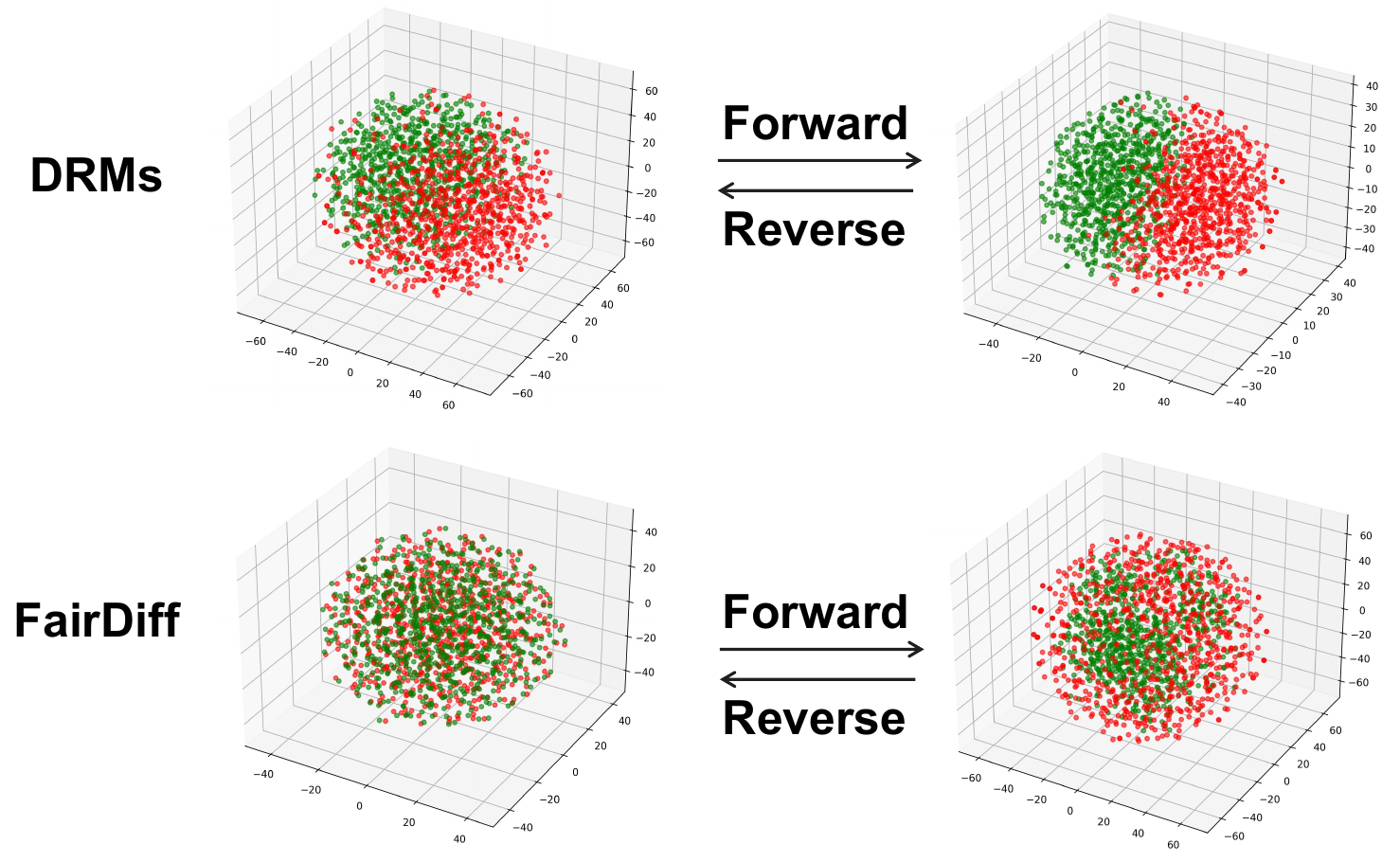}
    \caption{Visualization of item distributions in the forward–reverse process. Green denotes the item distribution, while red denotes the reconstruction distribution. The first column shows the forward initial and reverse final distributions; the second column shows the forward final and reverse initial distributions.}
    \label{fig:visual}
\end{figure}

\subsection{Visualization Study (RQ5)}

To examine how the SC module mitigates forward–reverse prior mismatch, we visualize 3D item embeddings during the diffusion process (Figure~\ref{fig:visual}), comparing DRMs with white-noise initialization and those equipped with semantic calibration. 
Under standard DRMs, unconstrained white noise produces disjoint and scattered clusters, indicating semantic distortion and a misalignment between the forward terminal distribution and the reverse initialization, which further induces a drift toward high-frequency items. 
In contrast, semantic calibration constrains the reverse initial noise to align with the intrinsic item manifold, resulting in compact and coherent embedding structures. 
This preference-aware initialization effectively bridges the forward–reverse gap, preserving semantic fidelity while alleviating popularity bias.
We further quantify this alignment using KL-divergence, a special case of the $f$-divergence defined as
$D_f(P \| Q) = \mathbb{E}_{Q}\!\left[f\!\left(\frac{P(x)}{Q(x)}\right)\right]$ and $\quad f(u)=u\log u$.
Estimated via empirical sampling, the divergence between the original and reconstructed distributions decreases from 8.2825 (DRMs) to 0.2383, while the divergence between the forward terminal and reverse initial distributions drops from 10.6447 to 0.2475. 
These results demonstrate that SC substantially improves forward–reverse consistency, thereby mitigating the self-reinforcing Matthew effect at the distributional level.

\section{Conclusion}

This work reveals that Diffusion Recommender Models (DRMs) suffer from a unique, self-reinforcing Matthew Effect, where popularity bias is not merely inherited from long-tailed data but actively amplified by the model's generative dynamics. We demonstrate that this phenomenon originates from a two-stage compounding mechanism: the universal challenge of popularity-dominated optimization, compounded by a DRM-specific structural prior mismatch during reverse generation. To dismantle this self-reinforcing loop, we propose FairDiff, a plug-and-play fairness-aware diffusion framework. Rather than attempting to alter the universally skewed training objective, FairDiff intervenes at inference time via Popularity Condition Guidance—a mathematically grounded distributional reweighting mechanism—and Semantic Calibration, which aligns the mismatched prior distributions via optimal transport. Extensive experiments demonstrate that FairDiff consistently improves recommendation accuracy while substantially reducing exposure disparity and enhancing long-tail coverage, proving that fairness and utility can be jointly achieved through the principled control of generative diffusion trajectories.

\bibliography{aaai2027}


\section{Details of Datasets}
\label{data}
The details of
these datasets are provided in Tables~\ref{dataset} and~\ref{dataset2}. All experiments are conducted on A100 GPUs.
\begin{table}[htbp]
    \centering
    \caption{Statistics of sequential recommendation datasets. Avg len means the average length of sequences.}
    \label{tab:dataset_stats}
    \small
    \begin{tabular}{l@{\hspace{1pt}}rrrrr}
        \toprule
        \textbf{Dataset} &  \textbf{\# Sequence} &  \textbf{\# items} &  \textbf{\# Actions} &  \textbf{Avg len} &  \textbf{Sparsity} \\
        \midrule
        Beauty & 22,363 & 12,101 & 198,502 & 8.53 & 99.93\% \\
        Toys & 19,412 & 11,924 & 167,597 & 8.63 & 99.93\% \\
        ML-1M & 6,040 & 3,416 & 999,611 & 165.50 & 95.16\% \\
        Steam & 281,428 & 13,044 & 3,485,022 & 12.40 & 99.90\% \\
        
        \bottomrule
    \end{tabular}
    \label{dataset}
\end{table}

\begin{table}[htbp]
    \centering
    \caption{Statistics of cross-domain  multimodal recommendation datasets. V, A, and T denote visual, acoustic, and textual modalities, respectively.}
    \label{dataset2}
    \small
    \setlength{\tabcolsep}{2pt}
    \renewcommand{\arraystretch}{1.2}
    \begin{tabular}{lccccc}
        \toprule
        \textbf{Dataset} & Modalities & \# Users & \# Items & \# Interactions & Sparsity \\
        \midrule
        TikTok & V, A, T & 9,319 & 6,710 & 59,541 & 99.90\% \\
        Baby & V, T & 19,445 & 7,050 & 139,110 & 99.90\% \\
        Sports & V, T & 35,598 & 18,357 & 256,308 & 99.96\% \\
        \bottomrule
    \end{tabular}
\end{table}

\begin{table}[t]
\centering
\caption{Efficiency Experiments.}
\label{Efficiency}
\resizebox{0.8\linewidth}{!}{%
\begin{tabular}{l|l|cccc}
\toprule
\textbf{Backbone} & \textbf{Datasets} & \textbf{Steps} & \textbf{Epochs}& \textbf{Training Time} & \textbf{Inference Time}\\
\midrule
\multirow{4}{*}{DreamRec} 
 & Beauty &350 & 248 & 369 minutes &344 seconds\\
 & Toys &323 & 365 & 472 minutes &218 seconds\\
 & Steam &5430 & 215 & 1850 minutes &2245 seconds\\
 & ML-1M &2649 & 301 & 1627 minutes &180 seconds\\
 \midrule
 \multirow{4}{*}{+FairDiff} 
 & Beauty &145 & 109 & 163 minutes & 149 seconds \\
 & Toys &131 & 153 & 195 minutes & 92 seconds \\
 & Steam &2184 & 84 & 743 minutes & 958 seconds   \\
 & ML-1M &1243 & 112 & 672 minutes & 73 seconds \\
\midrule
\multirow{4}{*}{DiffuRec} 
 & Beauty &327 & 211 & 344 minutes &298 seconds\\
 & Toys &286 & 339 & 450 minutes &174 seconds\\
 & Steam &5376 & 160 & 1709 minutes &2204 seconds\\
 & ML-1M &2626 & 271 & 1599 minutes &139 seconds\\
\midrule
 \multirow{4}{*}{+FairDiff} 
 & Beauty &187 & 117 & 201 minutes &194 seconds\\
 & Toys &190 & 171 & 257 minutes &114 seconds\\
 & Steam &3224 & 105 & 925 minutes &1108 seconds\\
 & ML-1M &1707 & 178 & 1058 minutes &82 seconds\\
\midrule
\multirow{4}{*}{CDiff4Rec} 
 & Beauty &257 & 174 & 313 minutes &251 seconds\\
 & Toys &241 & 295 & 412 minutes & 138 seconds\\
 & Steam &5312 & 125 & 1642 minutes & 2147 seconds \\
 & ML-1M &2594 & 240 & 1564 minutes & 96 seconds \\
 \midrule
  \multirow{4}{*}{+FairDiff} 
 & Beauty &122 & 81 & 95 minutes &63 seconds\\
 & Toys &162 & 114 & 208 minutes & 74 seconds\\
 & Steam &3124 & 69 & 726 minutes & 1382 seconds \\
 & ML-1M &1058 & 136 & 539 minutes & 69 seconds \\
\bottomrule
\end{tabular}
}
\end{table}

\section{More Experiments}

\begin{table*}[tbhp]
\centering
\small
\caption{Temporal Modeling Results between backbones and FairDiff on four datasets.}
\label{tab:backbones_pdrec}
\resizebox{\textwidth}{!}{
\begin{tabular}{c|c|cc|ccc|ccc|ccc}
\toprule
Datasets & Metrics & T-DiffRec & TI-DiffRec & GRU4Rec & +PDRec &\textbf{+FairDiff} & SASRec & +PDRec &\textbf{+FairDiff}& CL4SRec & +PDRec&\textbf{+FairDiff} \\
\midrule
\multicolumn{1}{c}{\multirow{5}{*}{Toy}} 
& N@1 & 0.1033 & 0.1058 & 0.0878 & 0.0899 & \textbf{0.1137}& 0.1095 & 0.1247 & \textbf{0.1426}&0.1125 & 0.1254 & \textbf{0.1575}\\
& N@5 & 0.1564 & 0.1618 & 0.1515 & 0.1617 & \textbf{0.1749}& 0.1779 & 0.2023& \textbf{0.2218} & 0.1802 & 0.2041 & \textbf{0.2471}\\
& N@10 & 0.1758 & 0.1823 & 0.1755 & 0.1879& \textbf{0.2114} & 0.2020 & 0.2286& \textbf{0.2371} & 0.2046 & 0.2305 & \textbf{0.2415}\\
& HR@5 & 0.2055 & 0.2151 & 0.2128 & 0.2300& \textbf{0.2451} & 0.2423 & 0.2752& \textbf{0.2913} & 0.2438 & 0.2776 & \textbf{0.3048}\\
& HR@10 & 0.2657 & 0.2787 & 0.2874 & 0.3112 & \textbf{0.3473}& 0.3169 & 0.3568 & \textbf{0.3764}& 0.3195 & 0.3595 & \textbf{0.3864}\\
& AUC & 0.5911 & 0.5968 & 0.5670 & 0.5909 & \textbf{0.6127}& 0.5771 & 0.6060 & \textbf{0.6239}& 0.5805 & 0.6068 & \textbf{0.6398}\\
\midrule
\multicolumn{1}{c}{\multirow{5}{*}{Game}} 
& N@1 & 0.1611 & 0.1746 & 0.1667 & 0.1808 & \textbf{0.2014}& 0.2111 & 0.2191 & \textbf{0.2479}& 0.2106 & 0.2180 & \textbf{0.2274}\\
& N@5 & 0.2567 & 0.2723 & 0.2818 & 0.2996 & \textbf{0.3258}& 0.3310 & 0.3382 & \textbf{03410}& 0.3294 & 0.3368 & \textbf{0.3527}\\
& N@10 & 0.2895 & 0.3040 & 0.3199 & 0.3380 & \textbf{0.3522}& 0.3682 & 0.3753& \textbf{0.3908} & 0.3682 & 0.3750 & \textbf{0.4001}\\
& HR@5 & 0.3451 & 0.3618 & 0.3893 & 0.4091 & \textbf{0.4326}& 0.4409 & 0.4475 & \textbf{0.4619}& 0.4385 & 0.4456 & \textbf{0.4633}\\
& HR@10 & 0.4469 & 0.4600 & 0.5071 & 0.5282 & \textbf{0.5642}& 0.5559 & 0.5626 & \textbf{0.5937}& 0.5584 & 0.5638 & \textbf{0.6024}\\
& AUC & 0.7217 & 0.7234 & 0.7601 & 0.7786 &\textbf{0.7941}& 0.7865 & 0.7898 &\textbf{0.8127}& 0.7857 & 0.7895 &\textbf{0.8230}\\
\midrule
\multicolumn{1}{c}{\multirow{5}{*}{Book}} 
& N@1 & 0.3194 & 0.3275 & 0.3299 & 0.3540 &\textbf{0.3769}& 0.3753 & 0.3826 &\textbf{0.4092}& 0.3689 & 0.3755 &\textbf{0.4271}\\
& N@5 & 0.4709 & 0.4773 & 0.4725 & 0.5000 &\textbf{0.5270}& 0.5170 & 0.5283 &\textbf{0.5314}& 0.5096 & 0.5211 &\textbf{0.5448}\\
& N@10 & 0.4987 & 0.5049 & 0.5069 & 0.5348 &\textbf{0.5526}& 0.5503 & 0.5620 &\textbf{0.5821}& 0.5435 & 0.5558 &\textbf{0.5827}\\
& HR@5 & 0.5852 & 0.5886 & 0.5987 & 0.6287 &\textbf{0.6455}& 0.6421 & 0.6573 &\textbf{0.6720}& 0.6353 & 0.6511 &\textbf{0.6648}\\
& HR@10 & 0.6706 & 0.6738 & 0.7048 & 0.7348 &\textbf{0.7426}& 0.7503 & 0.7620 &\textbf{0.7751}& 0.7435 & 0.7558 &\textbf{0.7742}\\
& AUC & 0.8329 & 0.8318 & 0.8768 & 0.8908 &\textbf{0.9136}& 0.8962 & 0.9040 &\textbf{0.9315}& 0.8939 & 0.9026 &\textbf{0.9371}\\
\bottomrule
\end{tabular}
}
\end{table*}
\begin{table*}[htbp]
\centering
\caption{Performance comparison of CDR to cold-start users in three CDR scenarios. }
\label{tab:cdr_performance}
\resizebox{\linewidth}{!}{ 
\begin{tabular}{l|cc|cc|cc|cc|cc|cc|cc|cc|cc}
\toprule
\multirow{2}{*}{\textbf{Methods}} & \multicolumn{6}{c|}{\textbf{Movie$\rightarrow$Music}} & \multicolumn{6}{c|}{\textbf{Book$\rightarrow$Movie}} & \multicolumn{6}{c}{\textbf{Book$\rightarrow$Music}} \\
\cline{2-19}
 & \multicolumn{2}{c|}{\textbf{20\%}} & \multicolumn{2}{c|}{\textbf{50\%}} & \multicolumn{2}{c|}{\textbf{80\%}} & \multicolumn{2}{c|}{\textbf{20\%}} & \multicolumn{2}{c|}{\textbf{50\%}} & \multicolumn{2}{c|}{\textbf{80\%}} & \multicolumn{2}{c|}{\textbf{20\%}} & \multicolumn{2}{c|}{\textbf{50\%}} & \multicolumn{2}{c}{\textbf{80\%}} \\
\cline{2-19}
 & \textbf{MAE} & \textbf{RMSE} & \textbf{MAE} & \textbf{RMSE} & \textbf{MAE} & \textbf{RMSE} & \textbf{MAE} & \textbf{RMSE} & \textbf{MAE} & \textbf{RMSE} & \textbf{MAE} & \textbf{RMSE} & \textbf{MAE} & \textbf{RMSE} & \textbf{MAE} & \textbf{RMSE} & \textbf{MAE} & \textbf{RMSE} \\
\midrule
TGT & 4.4803 & 5.1580 & 4.4989 & 5.1736 & 4.5020 & 5.1891 & 4.1831 & 4.7536 & 4.2288 & 4.7920 & 4.2123 & 4.8149 & 4.4873 & 5.1672 & 4.5073 & 5.1727 & 4.5024 & 5.2308 \\
CMF & 1.5209 & 2.0158 & 1.6893 & 2.2271 & 2.4186 & 3.0936 & 1.3632 & 1.7918 & 1.5813 & 2.0886 & 2.1577 & 2.6777 & 1.8284 & 2.3829 & 2.1282 & 2.7275 & 3.0130 & 3.6948 \\
DCDCSR & 1.4918 & 1.9210 & 1.8144 & 2.3439 & 2.7194 & 3.3065 & 1.3971 & 1.7346 & 1.6731 & 2.0551 & 2.3618 & 2.7702 & 1.8411 & 2.2955 & 2.1736 & 2.6771 & 3.1405 & 3.5842 \\
SSCDR & 1.3017 & 1.6579 & 1.3762 & 1.7477 & 1.5046 & 1.9229 & 1.2390 & 1.6526 & 1.2137 & 1.5602 & 1.3172 & 1.7024 & 1.5414 & 1.9283 & 1.4739 & 1.8441 & 1.6414 & 2.1403 \\
CATN & 1.2671 & 1.6468 & 1.4890 & 1.9205 & 1.8182 & 2.2991 & 1.1249 & 1.4548 & 1.1598 & 1.4826 & 1.2672 & 1.6280 & 1.3924 & 1.7399 & 1.6023 & 2.0600 & 1.9719 & 2.5623 \\
EMCDR & 1.2350 & 1.5515 & 1.3277 & 1.6644 & 1.5008 & 1.8771 & 1.1162 & 1.4120 & 1.1832 & 1.4981 & 1.3156 & 1.6433 & 1.3524 & 1.6777 & 1.4723 & 1.8665 & 1.7571 & 2.2119 \\
PTUPCDR & 1.1504 & 1.5195 & 1.2804 & 1.6380 & 1.4049 & 1.8234 & 0.9970 & 1.3317 & 1.0894 & 1.4395 & 1.1999 & 1.5916 & 1.2286 & 1.6085 & 1.3764 & 1.7447 & 1.5784 & 2.0510 \\
DiffCDR & 1.0435 & 1.3840 & 1.2367 & 1.6850 & 1.5606 & 2.1754 & 0.9476 & 1.2338 & 0.9953 & 1.3155 & 1.0846 & 1.4695 & 1.1720 & 1.5390 & 1.3077 & 1.8255 & 1.5871 & 2.2110 \\
REMIT & 0.9393 & 1.2709 & 1.0437 & 1.4589 & 1.2181 & 1.6601 & 0.8759 & 1.1650 & 0.9172 & 1.2379 & 0.8055 & 1.3772 & 1.3249 & 1.9940 & 1.4401 & 2.0495 & 1.6396 & 2.2653 \\
CDRNP & 0.7974 & 1.0638 & 0.7969 & 1.0589 & 0.8280 & 1.0758 & 0.8846 & 1.1327 & 0.8946 & 1.1450 & 0.8970 & 1.1576 & 0.7453 & 0.9914 & 0.7629 & 1.0111 & 0.7787 & 1.0373 \\

DMCDR & \underline{0.6173} & \underline{0.9425} & \underline{0.6015} & \underline{0.9559} & 0\underline{.6450} & \underline{0.9706} & \underline{0.5904} & \underline{0.9929} & \underline{0.6050} & \underline{0.9859} & \underline{0.6227} & \underline{1.0002} & \underline{0.5360} & \underline{0.8726} & \underline{0.5563} & \underline{0.8842} & \underline{0.5589} & \underline{0.8992} \\
\midrule
 \textbf{FairDiff} & \textbf{0.5827} & \textbf{0.9234} & \textbf{0.5837} & \textbf{0.9281} & \textbf{0.6341} & \textbf{0.9530} & \textbf{0.5738} & \textbf{0.9704} & \textbf{0.5898} & \textbf{0.9685} & \textbf{0.6040} & \textbf{0.9847} & \textbf{0.5124} & \textbf{0.8592} & \textbf{0.5299} & \textbf{0.8589} & \textbf{0.5219} & \textbf{0.8548} \\

\bottomrule
\end{tabular}
}
\end{table*}

\subsection{Efficiency}
We report the efficiency evaluation results of the proposed method in Table~\ref{Efficiency}, where all models are implemented and run on an A100 GPU.

\subsection{More Performance Comparison across items of varying popularity (RQ3)}

\begin{table*}[ht]
\centering
\small  
\caption{Performance across item popularity and user sequence length on sequential recommendation datasets.}
\label{tab:item_user_performance_more}
\setlength{\tabcolsep}{2pt} 
\renewcommand{\arraystretch}{1.2} 
\resizebox{\textwidth}{!}{%
\begin{tabular}{c|c|c|ccccccccccc|cc|cc|cc} 
\toprule[1.5pt]
\textbf{Dataset} & \textbf{Metric} & \textbf{Dimension} & \textbf{Popularity} &
\textbf{SASRec} & \textbf{DIN} & \textbf{BERT4Rec} & \textbf{ComiRec} & \textbf{TiMiRec} & \textbf{TIGER} & \textbf{BASRec} & \textbf{HSTU} & \textbf{ACVAE} & \textbf{CSRec} &
\textbf{DiffuRec} & \textbf{+FairDiff} & \textbf{DreamRec} & \textbf{+FairDiff} & \textbf{CDiff4Rec} & \textbf{+FairDiff} \\
\midrule
\multirow{12}{*}{\textbf{Steam}} & \multirow{6}{*}{H@20}
& \multirow{3}{*}{Item} & Cold & 0.32 & 0.41 & 0.25 & 0.21 & 0.38 & 0.45 & 0.52 & 0.50 & 0.36 & 0.55 & 0.61 & \textbf{0.82} & 0.78 & \textbf{0.95} & 0.85 & \textbf{1.02} \\
& & & Mid & 4.15 & 4.82 & 3.97 & 3.25 & 4.68 & 4.95 & 5.12 & 4.98 & 4.72 & 5.36 & 5.21 & \textbf{5.45} & 5.02 & \textbf{5.26} & 5.38 & \textbf{5.62} \\
& & & Hot & 15.28 & 16.95 & 14.87 & 13.65 & 16.25 & 15.82 & 16.53 & 16.10 & 15.76 & 17.23 & 17.85 & \textbf{19.12} & 19.08 & \textbf{20.23} & 18.95 & \textbf{21.05} \\
\cmidrule(lr){3-20}
& & \multirow{3}{*}{User} & Cold & 6.25 & 7.18 & 6.02 & 5.38 & 6.95 & 7.52 & 7.88 & 7.65 & 7.02 & 8.15 & 7.95 & \textbf{16.82} & 17.49 & \textbf{18.56} & 8.52 & \textbf{19.10} \\
& & & Mid & 12.35 & 13.18 & 11.97 & 10.85 & 12.88 & 13.52 & 13.85 & 13.62 & 12.95 & 14.23 & 14.58 & \textbf{16.95} & 17.36 & \textbf{17.52} & 15.10 & \textbf{18.02} \\
& & & Hot & 14.82 & 15.95 & 14.56 & 13.98 & 15.78 & 16.25 & 16.68 & 16.35 & 15.89 & 17.52 & 17.95 & \textbf{18.88} & 16.27 & \textbf{17.35} & 17.23 & \textbf{19.05} \\
\cmidrule(lr){2-20}
& \multirow{6}{*}{N@20} & \multirow{3}{*}{Item} & Cold & 0.15 & 0.21 & 0.12 & 0.10 & 0.18 & 0.22 & 0.25 & 0.23 & 0.16 & 0.26 & 0.28 & \textbf{0.35} & 0.33 & \textbf{0.41} & 0.36 & \textbf{0.45} \\
& & & Mid & 1.85 & 2.12 & 1.78 & 1.56 & 2.05 & 2.21 & 2.35 & 2.28 & 2.02 & 2.45 & 2.38 & \textbf{2.52} & 2.17 & \textbf{2.30} & 2.42 & \textbf{2.58} \\
& & & Hot & 6.58 & 7.25 & 6.32 & 5.89 & 7.05 & 7.38 & 7.65 & 7.42 & 7.18 & 8.02 & 7.85 & \textbf{8.52} & 8.20 & \textbf{8.80} & 8.35 & \textbf{9.02} \\
\cmidrule(lr){3-20}
& & \multirow{3}{*}{User} & Cold & 3.25 & 3.88 & 3.12 & 2.85 & 3.75 & 4.02 & 4.25 & 4.18 & 3.82 & 4.35 & 4.28 & \textbf{7.95} & 7.46 & \textbf{8.38} & 4.52 & \textbf{8.65} \\
& & & Mid & 5.82 & 6.55 & 5.68 & 5.25 & 6.38 & 6.72 & 6.95 & 6.82 & 6.52 & 7.15 & 7.08 & \textbf{7.35} & 7.41 & \textbf{7.51} & 7.22 & \textbf{7.68} \\
& & & Hot & 6.95 & 7.68 & 6.72 & 6.35 & 7.48 & 7.82 & 8.05 & 7.92 & 7.65 & 8.23 & 8.18 & \textbf{8.65} & 7.18 & \textbf{8.08} & 8.32 & \textbf{9.15} \\
\midrule
\multirow{12}{*}{\textbf{Beauty}} & \multirow{6}{*}{H@20} & \multirow{3}{*}{Item} & Cold & 1.05 & 1.32 & 0.98 & 0.85 & 1.25 & 1.42 & 1.58 & 1.52 & 1.18 & 1.65 & 1.52 & \textbf{2.05} & 1.88 & \textbf{2.25} & 2.21 & \textbf{2.45} \\
& & & Mid & 2.85 & 3.12 & 2.78 & 2.56 & 3.05 & 3.22 & 3.38 & 3.32 & 3.08 & 3.45 & 3.38 & \textbf{3.50} & 3.25 & \textbf{3.42} & 3.54 & \textbf{3.57} \\
& & & Hot & 12.58 & 13.25 & 11.97 & 10.85 & 12.88 & 13.52 & 14.05 & 13.82 & 13.18 & 14.52 & 14.38 & \textbf{16.82} & 15.95 & \textbf{17.23} & 17.08 & \textbf{19.14} \\
\cmidrule(lr){3-20}
& & \multirow{3}{*}{User} & Cold & 4.25 & 4.88 & 4.12 & 3.85 & 4.75 & 5.02 & 5.28 & 5.21 & 4.82 & 5.35 & 5.28 & \textbf{9.85} & 8.25 & \textbf{9.68} & 8.72 & \textbf{10.28} \\
& & & Mid & 8.95 & 9.68 & 8.72 & 8.35 & 9.48 & 9.82 & 10.05 & 9.92 & 9.65 & 10.23 & 10.18 & \textbf{11.52} & 9.88 & \textbf{11.25} & 10.80 & \textbf{11.83} \\
& & & Hot & 12.82 & 13.95 & 12.56 & 11.98 & 13.78 & 14.25 & 14.68 & 14.35 & 13.89 & 15.02 & 14.95 & \textbf{16.58} & 14.52 & \textbf{16.10} & 15.13 & \textbf{17.61} \\
\cmidrule(lr){2-20}
& \multirow{6}{*}{N@20} & \multirow{3}{*}{Item} & Cold & 0.65 & 0.82 & 0.62 & 0.55 & 0.78 & 0.88 & 0.95 & 0.92 & 0.75 & 1.02 & 0.98 & \textbf{1.15} & 0.95 & \textbf{1.12} & 1.20 & \textbf{1.40} \\
& & & Mid & 1.58 & 1.75 & 1.52 & 1.38 & 1.72 & 1.85 & 1.92 & 1.88 & 1.68 & 2.05 & 1.98 & \textbf{2.12} & 1.75 & \textbf{1.88} & 1.88 & \textbf{1.97} \\
& & & Hot & 5.82 & 6.55 & 5.68 & 5.25 & 6.38 & 6.72 & 6.95 & 6.82 & 6.52 & 7.15 & 7.08 & \textbf{8.25} & 7.88 & \textbf{8.32} & 8.51 & \textbf{9.55} \\
\cmidrule(lr){3-20}
& & \multirow{3}{*}{User} & Cold & 2.25 & 2.88 & 2.12 & 1.85 & 2.75 & 3.02 & 3.25 & 3.18 & 2.82 & 3.35 & 3.28 & \textbf{4.05} & 3.95 & \textbf{4.88} & 4.35 & \textbf{5.30} \\
& & & Mid & 4.82 & 5.55 & 4.68 & 4.25 & 5.38 & 5.72 & 5.95 & 5.82 & 5.52 & 6.15 & 6.08 & \textbf{5.75} & 5.68 & \textbf{5.82} & 5.39 & \textbf{5.89} \\
& & & Hot & 6.25 & 6.98 & 6.12 & 5.85 & 6.75 & 7.02 & 7.25 & 7.18 & 6.82 & 7.35 & 7.28 & \textbf{8.15} & 6.95 & \textbf{7.88} & 8.00 & \textbf{9.13} \\
\bottomrule[1.5pt]
\end{tabular}%
}
\end{table*}
Results are in Table~\ref{tab:item_user_performance_more}.

\subsection{Temporal Modeling in DRMs}
As shown in Table~\ref{tab:backbones_pdrec}, FairDiff consistently surpasses both diffusion-based baselines (T-DiffRec, TI-DiffRec) and diverse backbones (GRU4Rec, SASRec, CL4SRec) across datasets (Toy, Game, Book) and metrics (NDCG@k, HR@k, AUC). The gains are statistically significant (p<0.05), highlighting robustness and generality. Moreover, FairDiff adapts to dataset characteristics: on sparse data it improves hit rates, on medium-density data it enhances discrimination (AUC), and on dense corpora it sharpens ranking accuracy. These results demonstrate that FairDiff’s temporal modeling and distributional alignment mechanisms unlock the latent potential of heterogeneous recommendation architectures, establishing it as a universally effective enhancement for diffusion-based models.

\subsection{Cross-domain Recommendation Under Cold-start Scenarios}
We further evaluate our method on additional cross-domain datasets under cold-start scenarios, as shown in Table~\ref{tab:cdr_performance}. We compare FairDiff with the following competitive baselines. TGT~\cite{rendle2012bpr} represents collaborative filtering methods based on matrix factorization (MF), using only interactions from the target domain. CMF~\cite{singh2008relational} adopts collective matrix factorization to share user embeddings across source and target domains. EMCDR~\cite{man2017cross} is the first mapping-based method, which learns MF embeddings for users and items in both domains and transfers user preferences via a linear mapping function. DCDCSR~\cite{zhu2020deep} uses MF to learn latent factors and applies a deep neural network to map them across domains. CATN~\cite{zhao2020catn} is a review-based model that aligns aspect-level user–item correlations across domains. SSCDR~\cite{kang2019semi} extends EMCDR by introducing semi-supervised metric space mapping with multi-hop neighborhood inference. DiffCDR~\cite{xuan2024diffusion} employs diffusion models to generate user representations in the target domain without explicitly transferring personalized preferences. PTUPCDR~\cite{zhu2022personalized} follows a meta-learning paradigm and learns personalized mapping functions for each user. REMIT~\cite{sun2023remit} further enhances PTUPCDR by modeling multiple user interests with multiple mapping functions. CDRNP~\cite{li2024cdrnp} is a state-of-the-art CDR approach that adopts neural processes within a meta-learning framework to capture preference correlations among both overlapping and cold-start users.
As demonstrated in Table~\ref{tab:cdr_performance}, FairDiff achieves state-of-the-art performance across all three cross-domain recommendation scenarios (Movie→Music, Book→Movie, Book→Music) under varying cold-start sparsity levels (20\%, 50\%, 80\%). It consistently outperforms the strongest baseline DMCDR (second-best performer) in all 27 evaluation metrics (9 experimental settings × MAE/RMSE). 

\subsection{Self-Reinforcing Matthew Effect}

We empirically demonstrate the existence of the \emph{self-reinforcing Matthew Effect} in sequential recommendation.
As illustrated in Figure~\ref{matai}, across all compared models, the performance on cold items consistently deteriorates over training epochs, while that on hot items steadily improves.
This phenomenon reflects a cumulative advantage mechanism: items that initially receive higher exposure accumulate more interactions, which further amplifies their likelihood of being recommended in subsequent training cycles.
Notably, Diffusion Recommender Models (DRMs) exhibit a more pronounced pattern.
Beyond the standard popularity bias observed in conventional models, DRMs suffer from an additional \emph{two-stage bias}.
Specifically, the forward corruption process and the reverse denoising generation jointly amplify popularity-conditioned signals, causing errors on under-exposed items to accumulate over time.
As a result, DRM achieves the best performance on hot items but the worst on cold items among all baselines, indicating a continuous polarization effect during training.
In contrast, \textbf{FairDiff} effectively mitigates this self-reinforcing dynamic.
By explicitly regularizing popularity-conditioned generation and correcting exposure imbalance during the diffusion process, FairDiff stabilizes performance across popularity groups, preventing the progressive degradation on cold items while preserving competitive accuracy on hot items.
These results confirm that FairDiff not only alleviates static popularity bias, but also suppresses its self-reinforcing amplification throughout training.

\begin{figure*}[htbp]
  \centering
  \includegraphics[width=0.8\textwidth]{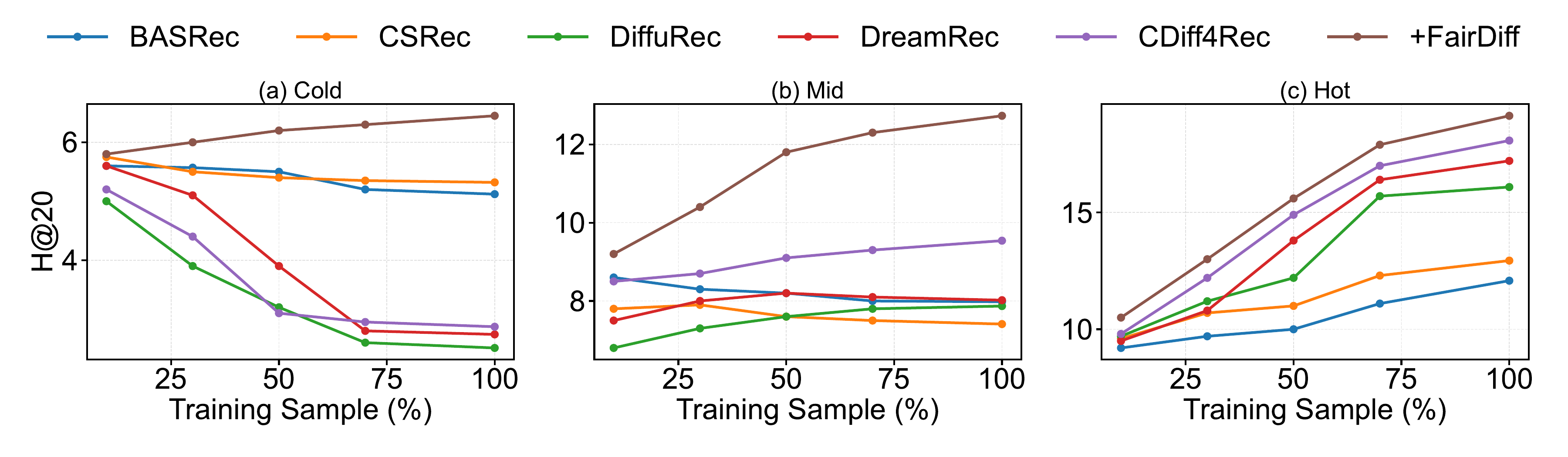}
  \caption{Self-Reinforcing Matthew Effect.}
  \label{matai}
\end{figure*}

\begin{table*}[htbp]
\centering
\small
\caption{Multimodal Recommendation Results.}
\setlength{\tabcolsep}{1.2pt} 
\resizebox{\textwidth}{!}{%
\begin{tabular}{c|c|ccccccccc|ccc|ccc}

\hline
\textbf{Dataset} & \textbf{Metric} & \textbf{BPR}& \textbf{LightGCN}&\textbf{VBPR}&\textbf{MMGCN}& \textbf{DualGNN}& \textbf{SLMRec}& \textbf{BM3}& \textbf{MGCN}& \textbf{Freedom}&  \textbf{DiffMM} & \textbf{+FairDiff} & \textbf{Imp.} & \textbf{DiffCL} & \textbf{+FairDiff} & \textbf{Imp.} \\
\hline
\multirow{6}{*}{TikTok} 
& H@5  & 0.013 & 0.022 & 0.017 & 0.025 & 0.027 & 0.028 & 0.031 & 0.034 & 0.035  & 0.037 & \textbf{0.053} & {+43.24\%} & 0.043 & \textbf{0.058} & {+34.9\%} \\
& H@10 &0.022&	0.031&	0.028&	0.038&	0.043&	0.044	&0.049	&0.054&	0.056 &   0.068 & \textbf{0.099} & {+45.6\%} & 0.075 & \textbf{0.108} & {+44.0\%} \\
& H@20 &0.040&	0.052&	0.049&	0.062&	0.068&	0.071	&0.078&	0.085&	0.089	& 0.110 & \textbf{0.128} & {+16.4\%} & 0.122 & \textbf{0.139} & {+13.9\%} \\
& N@5 &  0.007&	0.014&	0.010	&0.016	&0.017&	0.017	&0.019&	0.021&	0.022 & 0.025 & \textbf{0.041} & {+64.0\%} & 0.026 & \textbf{0.048} & {+84.6\%} \\
& N@10 &0.012 &	0.020 &	0.015 &	0.022 &	0.023 &	0.024 &	0.026 &	0.028 &	0.030 & 0.032 & \textbf{0.044} & {+37.5\%} & 0.037 &\textbf{ 0.055} & {+48.6\%} \\
& N@20 & 0.019&	0.028&	0.022&	0.031&	0.033	&0.034&	0.037&	0.040&	0.042 & 0.045 & \textbf{0.052} & {+15.6\%} & 0.049 & \textbf{0.067} & {+36.7\%} \\
\hline
\multirow{6}{*}{Sports} 
& H@5 &0.024&	0.034&	0.028	&0.033&	0.037&	0.039&	0.041&	0.043&	0.040 & 0.044 & \textbf{0.060} & {+36.4\%} & 0.047 & \textbf{0.073} & {+55.3\%} \\
& H@10& 0.041&	0.052&	0.048&	0.056&	0.057&	0.065&	0.061&	0.064&	0.060 & 0.068 & \textbf{0.083} & {+22.1\%} & 0.073 & \textbf{0.095} & {+30.1\%} \\
& H@20& 0.072& 	0.084& 	0.076& 	0.087& 	0.091& 	0.095& 	0.098& 	0.099& 	0.094 & 0.102 & \textbf{0.121} & +{18.6\%} & 0.110 & \textbf{0.134} & {+21.82\%} \\
& N@5&   0.009&	0.017&	0.012&	0.018	&0.023&	0.020&	0.021&	0.028&	0.025 & 0.029 & \textbf{0.035} & {+20.7\%} & 0.032 & \textbf{0.039} & {+21.9\%} \\
& N@10&  0.015&	0.021&	0.017	&0.022&	0.025&	0.028&	0.029&	0.031&	0.036& 0.037 & \textbf{0.044} & {+18.9\%} & 0.040 & \textbf{0.053} & {+32.5\%} \\
& N@20&  0.024 &	0.034 &	0.027 &	0.039 &	0.040 &	0.041 &	0.044 &	0.045	 &0.042 & 0.046 & \textbf{0.055} & {+19.6\%} & 0.051 & \textbf{0.059} & {+15.7\%} \\
\hline
\multirow{6}{*}{Baby} 
& H@5& 0.009	& 0.016	& 0.012& 	0.015&	0.019&	0.022&	0.024&	0.028	&0.029 & 0.031 & \textbf{0.050} & {+61.3\%} & 0.039 & \textbf{0.052} & {+33.3\%} \\
& H@10&   0.026&	0.034&	0.029&	0.037&	0.040&	0.041	&0.044&	0.045&	0.051 & 0.056 & \textbf{0.089} & {+58.9\%} & 0.064 & \textbf{0.094} & {+46.9\%} \\
& H@20& 0.056& 	0.065& 	0.059& 	0.069& 	0.073& 	0.076& 	0.078& 	0.082& 	0.089  & 0.096 & \textbf{0.114} & {+18.8\%} & 0.099 & \textbf{0.115} & {+16.2\%} \\
& N@5&   0.004&	0.010&	0.007&	0.011&	0.012&	0.014&	0.015&	0.018	&0.021& 0.023 & \textbf{0.033} & {+43.5\%} & 0.027 & \textbf{0.041} & {+51.9\%} \\
& N@10&   0.011	&0.018&	0.014&	0.020&	0.021&	0.028&	0.025&	0.031&	0.034	  & 0.039 & \textbf{0.042} & {+7.7\%} & 0.034 & \textbf{0.040} & {+17.6\%} \\
& N@20& 0.017&	0.025&	0.020&	0.023&	0.028&	0.031&	0.032&	0.035	&0.039 & 0.041 & \textbf{0.049} & {+19.5\%} & 0.043 & \textbf{0.052} & {+20.9\%} \\
\hline

\end{tabular}
}
\label{tab:additional_performance}
\end{table*}

\begin{table*}[htbp]
\centering
\small
\caption{Cross-domain Recommendation Results.}
\label{tab:ablation_cdsr}
\begin{tabular}{l|l|ccccccccccc}
\toprule
\textbf{Setting} & \textbf{Algorithm} & \textbf{N@1} & \textbf{N@5} & \textbf{N@10} & \textbf{N@20} & \textbf{N@50} & \textbf{HR@5} & \textbf{HR@10} & \textbf{HR@20} & \textbf{HR@50} & \textbf{AUC} \\
\midrule
\multirow{7}{*}{\textbf{Game $\rightarrow$ Toy} }& T-DiffRec (M)         & 0.0981 & 0.1520 & 0.1727 & 0.1934 & 0.2375 & 0.2029 & 0.2673 & 0.3494 & 0.5780 & 0.5924 \\
& TI-DiffRec (M)        & 0.1053 & 0.1598 & 0.1806 & 0.2008 & 0.2407 & 0.2111 & 0.2759 & 0.3562 & 0.5623 & 0.5932 \\
& SASRec (M)            & 0.1267 & 0.2019 & 0.2261 & 0.2490 & 0.2785 & 0.2722 & 0.3472 & 0.4380 & 0.5873 & 0.5951 \\
& PDRec(M)     & 0.1302 & 0.2093 & 0.2348 & 0.2574 & 0.2873 & 0.2826 & 0.3616 & 0.4515 & 0.6026 & 0.6106 \\
&+PCG & 0.1395 &0.2137&0.2497 &0.2674 &0.2935&0.3048&0.3784&0.4672&0.6173&0.6392\\
&+SC & 0.1382 &0.2125&0.2489 &0.2682 &0.2931&0.3045&0.3777&0.4668&0.6160&0.6386\\
&\textbf{+FairDiff(M)} & \textbf{0.1457} &\textbf{0.2247}&\textbf{0.2583}&\textbf{0.2738}&\textbf{0.3041}&\textbf{0.3148}&\textbf{0.3856}&\textbf{0.4756}&\textbf{0.6209}&\textbf{0.6412}\\
\midrule
\multirow{7}{*}{\textbf{Toy $\rightarrow$ Game}} & T-DiffRec (M)         & 0.1674 & 0.2643 & 0.2977 & 0.3247 & 0.3597 & 0.3548 & 0.4584 & 0.5655 & 0.7428 & 0.7232 \\
& TI-DiffRec (M)        & 0.1709 & 0.2757 & 0.3096 & 0.3378 & 0.3721 & 0.3723 & 0.4773 & 0.5887 & 0.7622 & 0.7407 \\
& SASRec (M)            & 0.2273 & 0.3532 & 0.3905 & 0.4190 & 0.4467 & 0.4674 & 0.5826 & 0.6955 & 0.8342 & 0.8007 \\
& PDRec(M)    & 0.2363 & 0.3623 & 0.3992 & 0.4275 & 0.4572 & 0.4761 & 0.5904 & 0.7022 & 0.8520 & 0.8153 \\
&+PCG & 0.2419 &0.3784 &0.4074 &0.4458 & 0.4638 &0.4893&0.6076 &0.7151 &0.8675 &0.8346\\
&+SC & 0.2411 &0.3778 &0.4063 &0.4463 & 0.4643 &0.4884&0.6071 &0.7148 &0.8679 &0.8342\\
&\textbf{+FairDiff(M)} & \textbf{0.2546} & \textbf{0.3853} &\textbf{0.4129}& \textbf{0.4627} & \textbf{0.4737} & \textbf{0.5002} & \textbf{0.6217} & \textbf{0.7314} & \textbf{0.8842} & \textbf{0.8416}\\
\bottomrule
\end{tabular}
\end{table*}

\subsection{Results of Multimodal Recommendation}  
We evaluate FairDiff on three multimodal sequential recommendation benchmarks and observe consistent gains over all baselines (Table~\ref{tab:additional_performance}), highlighting the transferability of its structured semantic priors to multimodal settings. By integrating the PCG and SC modules, FairDiff effectively incorporates multimodal contextual cues—ranging from visual embeddings to textual descriptions—into the diffusion process while preserving sequential dynamics, leading to more accurate alignment with user intent.
Remarkably, as we move from standard sequential recommendation tasks (Table~\ref{tab:performance}) to multimodal sequential tasks (Table~\ref{tab:additional_performance}), the performance advantage of FairDiff over baseline methods becomes increasingly pronounced with the addition of more modalities. This trend suggests that the structural distribution gap inherent in DRMs grows with modality complexity. FairDiff’s ability to effectively bridge this expanding gap under multimodal settings underscores its robustness and adaptability in capturing rich contextual signals.

\subsection{Results of Cross-domain Recommendation} 
FairDiff also demonstrates strong transferability in cross-domain recommendation (CDR). Following standard CDR protocols~\citep{ma2024triple,zheng2022ddghm}, we evaluate FairDiff on PDRec(M)~\citep{ma2024plug}, T-DiffRec(M)~\citep{ge2025time}, TI-DiffRec(M)~\citep{ma2024seedrec}, and SASRec(M)~\citep{kang2018self}, where (M) indicates direct chronological mixing of user behaviors from both source and target domains. Experiments are conducted on the Toy→Game and Game→Toy transfer settings. As shown in Table~\ref{tab:ablation_cdsr}, FairDiff consistently outperforms all diffusion-based baselines across CDR tasks, indicating its applicability beyond single-domain scenarios. Its PCG and SC offer an intuitive yet effective mechanism to inject semantic priors and amplify weak signals under domain shifts, paving the way for future research in this direction.
Furthermore, FairDiff surpasses all its ablated variants in most CDR settings, with each component contributing incremental gains. These results reaffirm the effectiveness and generalizability of PCG and SC within diffusion-based architectures.



\section{Theory}
\label{the}
\subsection{Matthew Effect in Diffusion Recommender Models}
\label{sec:drm_matthew}

We analyze the origin of the Matthew Effect in Diffusion Recommender Models (DRMs) from a probabilistic and dynamical perspective. In contrast to softmax-based recommenders, DRMs learn a generative score function through iterative denoising, where long-tail bias emerges intrinsically from the training objective and sampling process.

\subsubsection{Problem Formulation}

Let $\mathcal{D} = \{(u, i)\}$ denote the user--item interaction dataset, following a highly skewed long-tail distribution over items. Let $x_0^{(i)} \in \mathbb{R}^d$ denote the latent representation of item $i$, optionally conditioned on user $u$. A diffusion recommender model defines a forward noising process:
\begin{equation}
q(x_t \mid x_0) = \mathcal{N}(\sqrt{\alpha_t} x_0, (1-\alpha_t) I),
\end{equation}
and learns a reverse denoising process parameterized by a score network $s_\theta(x_t, t)$.

The training objective minimizes the expected denoising error:
\begin{align}
\mathcal{L}(\theta) =& \mathbb{E}_{(u,i) \sim \mathcal{D}} 
\mathbb{E}_{t \sim \mathcal{U}[1,T]}\notag\\
&
\mathbb{E}_{x_t \sim q(x_t \mid x_0^{(i)})}
\left[
\| s_\theta(x_t, t) - \nabla_{x_t} \log q(x_t \mid x_0^{(i)}) \|^2
\right].
\label{eq:ddpm_loss}
\end{align}

\subsubsection{Bias I: Gradient Dominance under Unconditional Denoising}

We first show that the unconditional (or weakly conditioned) denoising objective induces a gradient dominance effect favoring head items.

\begin{definition}[Item Frequency]
Let $f(i)$ denote the empirical frequency of item $i$ in the training set:
\begin{align}
f(i) = \sum_{u} \mathbb{I}[(u,i) \in \mathcal{D}].
\end{align}
\end{definition}

\begin{proposition}[Gradient Dominance by Head Items]
\label{prop:gradient_dominance}
Under the objective in Eq.~\ref{eq:ddpm_loss}, the expected gradient contribution of item $i$ to the score network satisfies:
\[
\mathbb{E}[\|\nabla_\theta \mathcal{L}_i\|] \propto f(i),
\]
where $\mathcal{L}_i$ denotes the loss term associated with item $i$.
\end{proposition}

\begin{proof}
Consider a single Monte Carlo step of stochastic gradient descent. A training sample $(u,i)$ is drawn uniformly from $\mathcal{D}$, which induces the marginal sampling probability $p(i)$ for item $i$.

Let $\ell(x_t^{(i)}, t)$ denote the inner denoising loss for a fixed item $i$, time step $t$, and noise realization. By linearity of expectation, the expected gradient of the overall objective can be decomposed as
\begin{align}
\mathbb{E}\bigl[ \nabla_\theta \mathcal{L}(\theta) \bigr]
&=
\sum_{i} p(i),
\mathbb{E}!\left[
\nabla_\theta \ell(x_t^{(i)}, t)
;\middle|; i
\right].
\end{align}

Since the objective is unconditional (or weakly conditioned), there exists no item-dependent reweighting or normalization term that counteracts $p(i)$. Therefore, the expected frequency with which gradients associated with item $i$ are accumulated is proportional to $f(i)$.

Taking norms on both sides yields
\[
\mathbb{E}!\left[ \bigl| \nabla_\theta \mathcal{L}_i \bigr| \right]
\propto f(i),
\]
up to a constant determined by the conditional denoising difficulty of item $i$, completing the proof.
\end{proof}

\begin{corollary}[Semantic Collapse of Tail Items]
For tail items with small $f(i)$, the learned score function $s_\theta$ poorly approximates $\nabla_{x_t} \log q(x_t \mid x_0^{(i)})$, causing their semantic signal to be rapidly overwhelmed by noise during forward diffusion and irrecoverable during iterative denoising.
\end{corollary}

This establishes the first source of the Matthew Effect: \emph{a training-time bias embedded in the learned diffusion score dynamics}.

 \subsubsection{Bias II: Terminal Prior Mismatch in Reverse Sampling}

We now identify a second, orthogonal source of the Matthew Effect, arising from a mismatch between the terminal distribution induced by the forward diffusion process and the initialization distribution used in reverse sampling. Unlike Bias I, which operates at the level of optimization dynamics, this bias emerges during generation and persists even with a perfectly trained score network.

\paragraph{Terminal Distribution Induced by Long-Tailed Data}

\begin{definition}[Terminal Distribution Coverage]
Let $q_T(x)$ denote the marginal distribution of the forward diffusion process at time step $T$:
\begin{equation}
q_T(x)
\sum_{i} p(i), q(x_T \mid x_0^{(i)}),
\end{equation}
where $p(i)$ denotes the empirical item prior and
\[
q(x_T \mid x_0^{(i)}) = \mathcal{N}(\sqrt{\alpha_T} x_0^{(i)}, (1-\alpha_T) I).
\]
\end{definition}

Since $\alpha_T \ll 1$, each component distribution is centered near the origin but remains distinguishable in high dimensions. Under a long-tailed item prior, $q_T(x)$ forms a highly imbalanced Gaussian mixture in which the total probability mass is dominated by head items, while tail-item components occupy vanishingly small measure.

\begin{theorem}[Reverse Initialization Bias]
\label{thm:terminal_mismatch}
Assume the score network $s_\theta$ accurately approximates $\nabla_x \log q_t(x)$ for all $t \in [0,T]$. When reverse sampling is initialized from an isotropic Gaussian prior
\[
p_T(x) = \mathcal{N}(0,I),
\]
the probability that a reverse trajectory converges to a tail item $i$ is upper-bounded by
\[
\mathbb{P}(\text{generate } i)
;\le;
C \cdot p(i),
\]
where $C$ is a constant independent of $i$. Consequently, as $p(i) \to 0$, the probability of generating tail items is linearly—and effectively exponentially in practice—suppressed relative to head items.
\end{theorem}

\begin{proof}

Consider the probability flow ODE corresponding to the reverse diffusion process:
\begin{equation}
\frac{d x_t}{d t}
* \frac{1}{2} \beta_t
  \left(
  x_t + 2 \nabla_x \log q_t(x_t)
  \right),
\end{equation}
  which defines a deterministic mapping $\Phi: x_T \mapsto x_0$ almost everywhere.

For each item $i$, define its attraction basin at time $T$ as
\[
\mathcal{A}_i
\left\{
x_T \in \mathbb{R}^d ;:; \Phi(x_T) \text{ converges to } x_0^{(i)}
\right\}.
\]

Under exact score estimation, the probability that a reverse trajectory terminates at item $i$ equals the measure of $\mathcal{A}_i$ under the initialization distribution:
\begin{equation}
\mathbb{P}(\text{generate } i)
\int_{\mathcal{A}_i} p_T(x), dx.
\end{equation}

Since $\Phi$ preserves probability mass along the probability flow, the basin $\mathcal{A}*i$ must satisfy
\[
\int*{\mathcal{A}_i} p_T(x), dx
 \int_{\Phi(\mathcal{A}_i)} q_0(x), dx
p(i),
\]
up to approximation error induced by finite $T$.

However, under a long-tailed prior, the basins ${\mathcal{A}_i}$ are embedded in regions where $q_T(x)$ differs substantially from the isotropic prior $p_T(x)$. In particular, for tail items with small $p(i)$, $\mathcal{A}*i$ occupies a proportionally small volume under $p_T(x)$, yielding
\[
\int*{\mathcal{A}_i} p_T(x), dx
\le
C \cdot p(i),
\]
for some constant $C$ determined by the overlap between $p_T$ and $q_T$.

Therefore, as $p(i) \to 0$, the probability of initializing a reverse trajectory within the attraction basin of a tail item vanishes, preventing successful recovery during reverse diffusion. This completes the proof.
\hfill$\square$
\end{proof}

\subsubsection{Orthogonality and Compositionality of Biases}

We now formalize the relationship between the two identified biases and show that the Matthew Effect in Diffusion Recommender Models arises from their composition rather than from a failure at any single stage.

\begin{lemma}[Persistence of Optimization Bias]
\label{lem:opt_bias_irreducible}
Consider a diffusion recommender model trained under the unconditional denoising objective in Eq.~\eqref{eq:ddpm_loss}. Even if reverse sampling is performed using the exact terminal distribution $q_T(x)$, the learned score function $s_\theta$ remains biased toward head items in expectation.
\end{lemma}

\begin{proof}
By Proposition~\ref{prop:gradient_dominance}, the expected gradient contribution of item $i$ during training is proportional to its empirical frequency $f(i)$. Consequently, the learned score function minimizes the denoising error primarily over high-density regions associated with head items.

Replacing the reverse initialization distribution $p_T(x)$ with the true terminal distribution $q_T(x)$ affects only the sampling stage and does not alter the training dynamics or gradient allocation. Therefore, the optimization process still yields a score function whose approximation error is systematically smaller for head items than for tail items, implying a persistent optimization bias independent of the sampling strategy.
\end{proof}

\begin{lemma}[Persistence of Sampling Bias]
\label{lem:sampling_bias_irreducible}
Assume access to an oracle score function $s^\ast(x,t) = \nabla_x \log q_t(x)$. If reverse sampling is initialized from an isotropic prior $p_T(x)=\mathcal{N}(0,I)$, tail items remain under-generated due to terminal prior mismatch.
\end{lemma}

\begin{proof}
Under the oracle score assumption, optimization bias is eliminated by construction. However, by Theorem~\ref{thm:terminal_mismatch}, the probability that a reverse trajectory initialized from $p_T(x)$ lies within the attraction basin of a tail item is proportional to its prior mass $p(i)$.

Since $p(i)$ is vanishingly small for tail items, their attraction basins occupy negligible measure under $p_T(x)$, resulting in persistent under-sampling. This phenomenon arises purely from the initialization mismatch and persists independently of score estimation accuracy.
\end{proof}

\begin{theorem}[Orthogonal and Compositional Biases in DRMs]
\label{thm:orthogonal_bias}
In diffusion recommender models trained on long-tailed data, the Matthew Effect arises from the composition of two orthogonal biases:
\begin{enumerate}
\item an \emph{optimization bias} induced by unconditional denoising during training;
\item a \emph{sampling bias} induced by terminal prior mismatch during reverse generation.
\end{enumerate}
These biases operate at distinct stages of the modeling pipeline and are not reducible to one another.
\end{theorem}

\begin{proof}
Orthogonality follows from Lemmas~\ref{lem:opt_bias_irreducible} and~\ref{lem:sampling_bias_irreducible}, which demonstrate that each bias persists when the other is eliminated via a counterfactual intervention (oracle sampling or oracle training).

Let $\pi_{\text{gen}}(i)$ denote the marginal probability that item $i$ is generated by the DRM. Under mild regularity assumptions, $\pi_{\text{gen}}(i)$ can be factorized as
\[
\pi_{\text{gen}}(i)
\mathbb{E}*{x_T \sim p_T}
\left[
\mathbb{P}\bigl( \Phi*\theta(x_T) = i \mid x_T \bigr)
\right],
\]
where $\Phi_\theta$ denotes the reverse mapping induced by the learned score network.

The optimization bias skews the conditional probability
$\mathbb{P}( \Phi_\theta(x_T) = i \mid x_T )$ toward head items, while the sampling bias suppresses the probability mass of $x_T$ falling into tail-item attraction basins. Since these effects act on different terms of the decomposition, their combined impact is multiplicative rather than substitutive.

Therefore, popularity amplification in DRMs is the result of the compounded effect of optimization and sampling biases, rather than an artifact of a single modeling stage. This completes the proof.
\hfill$\square$
\end{proof}
\begin{corollary}
In long-tailed recommendation settings, diffusion recommender models intrinsically amplify item popularity, even under idealized training or sampling conditions. Mitigating the Matthew Effect thus requires simultaneous correction of both optimization- and sampling-stage biases.
\end{corollary}

\subsection{Formalizing the Popularity-Dominated Diffusion Objective}

In this section, we formalize how the unified diffusion objective in DRMs becomes inherently dominated by item popularity, leading to weakened gradient signals for niche items.

Let \( \mathcal{D} = {(u,i)} \) denote the user–item interaction dataset, and let \( \mathbf{x}_0^i \in \mathbb{R}^d \) denote the latent representation of item \( i \). We define the empirical data distribution over items as
\[
p_{\text{data}}(\mathbf{x}*0) = \sum*{i \in \mathcal{I}} p(i),\delta(\mathbf{x}_0 - \mathbf{x}_0^i),
\]
where \( p(i) \propto c_i \) is proportional to the interaction count \( c_i \) of item \( i \), yielding a long-tailed popularity distribution.

Under the standard diffusion formulation, the training objective minimizes the expected denoising score matching loss:
\begin{equation}
\mathcal{L}_{\text{DSM}}
= \mathbb{E}_{i \sim p(i)}
\mathbb{E}_{t \sim \mathcal{U}[1,T]}
\mathbb{E}_{\boldsymbol{\epsilon} \sim \mathcal{N}(\mathbf{0},\mathbf{I})}
\Big[
\big|
\boldsymbol{\epsilon}
_{\boldsymbol{s}_\theta(\mathbf{x}_t^i, t)}
  \big|^2
  \Big],
\end{equation}
  where \( \mathbf{x}_t^i = \sqrt{\lambda_t}\mathbf{x}*0^i + \sqrt{1-\lambda_t}\boldsymbol{\epsilon} ) and ( \boldsymbol{s}*\theta \) denotes the score network.

Crucially, this objective is *implicitly weighted by item popularity* through \( p(i) \). The expected gradient contribution of item \( i \) thus satisfies
\begin{equation}
\mathbb{E}\big[|\nabla_\theta \mathcal{L}_i|\big]
\propto p(i),
\end{equation}
implying that popular items dominate the optimization dynamics, while niche items receive vanishingly small gradient signals.

We define the **popularity-dominated diffusion gap** as the variance of per-item gradient magnitudes:
\begin{equation}
\Delta_{\text{pop}}
:=
\text{Var}_{i \sim p(i)}
\left[
\mathbb{E}*t |\nabla*\theta \mathcal{L}_i|
\right],
\end{equation}
which quantifies the imbalance of learning across popularity groups.

\begin{proposition}
If \( p(i) \) follows a long-tailed distribution, then
\[
\Delta_{\text{pop}} \geq \Omega!\left(\text{Var}*{i}[p(i)]\right),
\]
and the gradient norm for tail items satisfies
\[
\lim*{c_i \to 0}
\mathbb{E}\big[|\nabla_\theta \mathcal{L}_i|\big] = 0.
\]
\end{proposition}

This result indicates that, during forward noising and iterative reverse denoising, representations of niche items are progressively overwhelmed by noise without sufficient corrective gradients. Consequently, their semantic structure collapses toward high-density regions dominated by popular items, degrading recommendation accuracy and reinforcing popularity bias at the diffusion-process level.
\subsection{Formalizing the Prior Mismatch}
In this section, we formalize the structural distribution gap in diffusion-based recommender models. 
Let \( p_{\text{data}} \) denotes the empirical distribution of latent variables \( \mathbf{x}_0 \in \mathbb{R}^d \). The forward diffusion process defines a Markov chain that progressively corrupts \( \mathbf{x}_0 \) over \( T \) steps, yielding a terminal distribution:
\begin{equation}
\mathbf{x}_T = \int_{\mathbb{R}^d} q(\mathbf{x}_T \mid \mathbf{x}_0)  p_{\text{data}}(\mathbf{x}_0)  d\mathbf{x}_0,
\end{equation}
which represents the pushforward of \(\hat{\mathbf{x}}_T\) under the diffusion kernel. In contrast, the reverse process is typically initialized from an isotropic Gaussian prior \( \hat{\mathbf{x}}_T = \mathcal{N}(\mathbf{0}, \mathbf{I}) \), which assumes maximal entropy and zero semantic content. The structural distribution gap is defined as the squared 2-Wasserstein distance between these distributions:
\begin{equation}
\Delta_{\text{struct}} := W_2^2( \mathbf{x}_T,  \hat{\mathbf{x}}_T ).
\end{equation}
This metric is chosen over f-divergences (e.g., KL divergence) for its sensitivity to both mass transportation and geometric support mismatch. While KL divergence is ill-defined for non-overlapping supports and primarily measures local density ratios, \( W_2 \) captures the minimal effort required to transport the mass of \( \mathbf{x}_T \) to \( \hat{\mathbf{x}}_T \) under a quadratic cost—directly reflecting the geometric distortion induced by popularity bias.
\begin{proposition}
Let \( \boldsymbol{\mu}_0 = \mathbb{E}_{p_{\text{data}}}[\mathbf{x}_0] \) and \( \boldsymbol{\Sigma}_0 = \text{Cov}_{p_{\text{data}}}[\mathbf{x}_0] \). Under a variance-preserving diffusion schedule \(\{\beta_t\}\), the following lower bound holds:
\begin{equation}
\Delta_{\text{struct}} \geq \lambda_T \|\boldsymbol{\mu}_0\|^2 + \text{tr}(\boldsymbol{\Sigma}_0) - d - 2\sqrt{\lambda_T}  \text{tr}(\boldsymbol{\Sigma}_0^{1/2}) + \mathcal{O}(1 - \lambda_T),
\end{equation}
where \( \lambda_T = \prod_{s=1}^T (1 - \beta_s) \). Moreover, if \( p_{\text{data}} \) is long-tailed, then
$\liminf_{T \to \infty} \Delta_{\text{struct}} \geq \text{tr}(\boldsymbol{\Sigma}_0) - d$.

\end{proposition}

Even asymptotically, the gap persists due to covariance mismatch induced by the geometric skew of interaction data. Popular items dominate \( \boldsymbol{\Sigma}_0 \), inflating its trace and biasing the reverse process toward high-frequency patterns. Consequently, cold items in low-density regions are systematically undersampled, causing representational collapse and reduced fairness.

\subsection{Theorem in Semantic Calibration}
\label{sc:them}
\subsubsection{Probability Flow ODE for Reverse Diffusion}
\begin{theorem}[Probability Flow ODE for Reverse Diffusion]
\label{thm:pf_ode}
Consider the forward diffusion process defined by the Itô SDE
\begin{equation}
d\mathbf{x}
=
\mathbf{f}_t(\mathbf{x})\,dt
+
g_t\,d\mathbf{w}_t,
\end{equation}
where $\mathbf{w}_t$ is a standard Wiener process.
Assume that the marginal density $p_t(\mathbf{x}\mid\mathbf{c})$ exists and is
continuously differentiable.
Then there exists a deterministic ordinary differential equation
\begin{equation}
\label{eq:pf_ode}
\frac{d\mathbf{x}(t)}{dt}
=
\mathbf{f}_t(\mathbf{x}(t))
-
\frac{1}{2} g_t^2
\nabla_{\mathbf{x}} \log p_t(\mathbf{x}(t)\mid\mathbf{c}),
\end{equation}
whose solution induces the same marginal distributions $\{p_t\}_{t\in[0,T]}$
as the original stochastic diffusion process.
\end{theorem}

\begin{proof}
The forward SDE induces a time-evolving density $p_t(\mathbf{x}\mid\mathbf{c})$
that satisfies the Fokker--Planck equation
\begin{equation}
\label{eq:fp}
\partial_t p_t
=
-\nabla_{\mathbf{x}}\!\cdot\!\big(\mathbf{f}_t(\mathbf{x}) p_t\big)
+
\frac{1}{2} g_t^2 \Delta_{\mathbf{x}} p_t .
\end{equation}

Now consider a deterministic flow defined by the ODE
\begin{equation}
\label{eq:ode_flow}
\frac{d\mathbf{x}(t)}{dt} = \mathbf{v}_t(\mathbf{x}(t)),
\end{equation}
whose induced density $\tilde{p}_t$ satisfies the continuity equation
\begin{equation}
\label{eq:continuity}
\partial_t \tilde{p}_t
=
-\nabla_{\mathbf{x}}\!\cdot\!\big(\mathbf{v}_t(\mathbf{x}) \tilde{p}_t\big).
\end{equation}

Setting
\begin{equation}
\mathbf{v}_t(\mathbf{x})
=
\mathbf{f}_t(\mathbf{x})
-
\frac{1}{2} g_t^2
\nabla_{\mathbf{x}} \log p_t(\mathbf{x}\mid\mathbf{c}),
\end{equation}
and substituting into Eq.~\eqref{eq:continuity}, we obtain
\begin{align}
\partial_t \tilde{p}_t
&=
-\nabla_{\mathbf{x}}\!\cdot\!\big(\mathbf{f}_t p_t\big)
+
\frac{1}{2} g_t^2
\nabla_{\mathbf{x}}\!\cdot\!\big(p_t \nabla_{\mathbf{x}} \log p_t\big) \\
&=
-\nabla_{\mathbf{x}}\!\cdot\!\big(\mathbf{f}_t p_t\big)
+
\frac{1}{2} g_t^2 \Delta_{\mathbf{x}} p_t,
\end{align}
which coincides with the Fokker--Planck equation in Eq.~\eqref{eq:fp}.
Therefore, the ODE in Eq.~\eqref{eq:pf_ode} induces the same marginal
distribution as the original stochastic diffusion process.
\end{proof}

\paragraph{Interpretation.}
The probability flow ODE replaces stochastic diffusion paths with a
deterministic transport whose velocity field depends on the score function.
While individual trajectories differ, the induced marginal distributions
are identical, enabling deterministic sampling and analysis.

\subsubsection{Score Matching as Kinetic Energy Minimization}
\begin{theorem}[Score Matching as Kinetic Energy Minimization]
\label{thm:score_energy}
Consider the probability flow ODE associated with a diffusion model,
\begin{equation}
\frac{d\mathbf{x}(t)}{dt} = \mathbf{v}_t(\mathbf{x}(t)),
\end{equation}
where $\mathbf{v}_t(\mathbf{x})$ induces a marginal density $p_t(\mathbf{x}\mid\mathbf{c})$
through the continuity equation
\begin{equation}
\label{eq:continuity_eq}
\partial_t p_t
+
\nabla_{\mathbf{x}}\!\cdot\!\big(p_t \mathbf{v}_t\big)
= 0.
\end{equation}
Then training a diffusion model via score matching is equivalent to learning
a velocity field $\mathbf{v}_t$ that minimizes the time-integrated kinetic energy
\begin{equation}
\label{eq:kinetic_energy}
\min_{\{\mathbf{v}_t\}}
\int_0^T
\mathbb{E}_{\mathbf{x}\sim p_t}
\big[
\|\mathbf{v}_t(\mathbf{x})\|_2^2
\big]\,dt,
\end{equation}
subject to the continuity constraint in Eq.~\eqref{eq:continuity_eq}.
\end{theorem}

\begin{proof}
For the forward diffusion SDE
\begin{equation}
d\mathbf{x}
=
\mathbf{f}_t(\mathbf{x})\,dt
+
g_t\,d\mathbf{w}_t,
\end{equation}
the associated probability flow ODE is given by
\begin{equation}
\label{eq:velocity_def}
\mathbf{v}_t(\mathbf{x})
=
\mathbf{f}_t(\mathbf{x})
-
\frac{1}{2} g_t^2
\nabla_{\mathbf{x}} \log p_t(\mathbf{x}\mid\mathbf{c}).
\end{equation}

Substituting Eq.~\eqref{eq:velocity_def} into the kinetic energy functional
in Eq.~\eqref{eq:kinetic_energy} yields
\begin{equation}
\int_0^T
\mathbb{E}_{\mathbf{x}\sim p_t}
\big[
\|\mathbf{f}_t(\mathbf{x})\|_2^2
\big]\,dt
+
\frac{1}{4}
\int_0^T
g_t^4
\,
\mathbb{E}_{\mathbf{x}\sim p_t}
\big[
\|
\nabla_{\mathbf{x}} \log p_t(\mathbf{x}\mid\mathbf{c})
\|_2^2
\big]\,dt,
\end{equation}
where cross terms vanish under integration by parts.
Since $\mathbf{f}_t$ and $g_t$ are fixed by the forward process, minimizing
the kinetic energy is equivalent to minimizing
\begin{equation}
\int_0^T
\mathbb{E}_{\mathbf{x}\sim p_t}
\big[
\|
\nabla_{\mathbf{x}} \log p_t(\mathbf{x}\mid\mathbf{c})
\|_2^2
\big]\,dt,
\end{equation}
which is precisely the objective minimized by score matching
up to a constant factor.
Therefore, learning the score function via diffusion training is equivalent
to learning a minimum-energy velocity field that transports the distribution
along the probability flow.
\end{proof}

\subsubsection{Score Approximation Induces Transport Deviation}
\begin{theorem}[Score Approximation Induces Transport Deviation]
\label{thm:score_ot_deviation}
Let $\mathbf{v}_t^\star$ denote the optimal velocity field associated with the true
score $\nabla_{\mathbf{x}}\log p_t(\mathbf{x}\mid\mathbf{c})$, and let
$\mathbf{v}_t^\theta$ be the velocity field induced by a learned score
$\nabla_{\mathbf{x}}\log p_t^\theta(\mathbf{x}\mid\mathbf{c})$.
Assume both velocity fields satisfy the continuity equation and are square-integrable.
Then any non-zero score approximation error induces a deviation from the optimal
transport trajectory, quantified by an excess kinetic energy:
\begin{equation}
\label{eq:excess_energy}
\int_0^T
\mathbb{E}_{\mathbf{x}\sim p_t}
\big[
\|\mathbf{v}_t^\theta(\mathbf{x})\|_2^2
\big]\,dt
>
\int_0^T
\mathbb{E}_{\mathbf{x}\sim p_t}
\big[
\|\mathbf{v}_t^\star(\mathbf{x})\|_2^2
\big]\,dt,
\end{equation}
unless $\mathbf{v}_t^\theta = \mathbf{v}_t^\star$ almost everywhere.
\end{theorem}

\begin{proof}
By the Benamou--Brenier formulation of optimal transport,
the optimal transport between two distributions over $[0,T]$ is realized by the
velocity field $\mathbf{v}_t^\star$ that minimizes the kinetic energy functional
\begin{equation}
\mathcal{E}(\mathbf{v})
=
\int_0^T
\mathbb{E}_{\mathbf{x}\sim p_t}
\big[
\|\mathbf{v}_t(\mathbf{x})\|_2^2
\big]\,dt,
\end{equation}
subject to the continuity equation.

From Theorem~\ref{thm:score_energy}, $\mathbf{v}_t^\star$ is uniquely determined
(up to measure-zero sets) by the true score
$\nabla_{\mathbf{x}}\log p_t(\mathbf{x}\mid\mathbf{c})$.
Let $\mathbf{v}_t^\theta = \mathbf{v}_t^\star + \boldsymbol{\delta}_t$ denote the
velocity induced by an approximate score, where $\boldsymbol{\delta}_t \neq \mathbf{0}$
on a set of positive measure.

Substituting into the kinetic energy yields
\begin{align}
\mathcal{E}(\mathbf{v}^\theta)
&=
\mathcal{E}(\mathbf{v}^\star)
+
\int_0^T
\mathbb{E}_{\mathbf{x}\sim p_t}
\big[
\|\boldsymbol{\delta}_t(\mathbf{x})\|_2^2
\big]\,dt
+
2 \int_0^T
\mathbb{E}_{\mathbf{x}\sim p_t}
\big[
\langle \mathbf{v}_t^\star, \boldsymbol{\delta}_t \rangle
\big]\,dt.
\end{align}

The cross term vanishes due to the optimality of $\mathbf{v}_t^\star$ as a minimizer
of $\mathcal{E}$ under the continuity constraint.
Therefore,
\begin{equation}
\mathcal{E}(\mathbf{v}^\theta)
>
\mathcal{E}(\mathbf{v}^\star),
\end{equation}
unless $\boldsymbol{\delta}_t = \mathbf{0}$ almost everywhere.
This implies that any non-zero score approximation error induces a deviation from
the optimal transport trajectory.
\end{proof}

\subsection{Theoretical Analysis of PCG}

In this section, we provide a rigorous theoretical justification for the proposed Popularity Condition Guidance (PCG). We demonstrate that PCG is not merely a heuristic inference trick but mathematically equivalent to sampling from a reweighted distribution that explicitly modulates popularity bias.

\subsubsection{Preliminaries and Score Decomposition}
Let the forward diffusion process be governed by a standard variance-preserving Stochastic Differential Equation (SDE). The marginal distribution at time $t$ is denoted by $p_t(\mathbf{x})$. The standard score matching objective minimizes the divergence between the estimated score $s_\theta(\mathbf{x}_t, t)$ and the true score $\nabla_{\mathbf{x}_t} \log p_t(\mathbf{x}_t)$.

We consider the item popularity $c \in \mathcal{C}$ as an observable conditioning variable. Following Bayes' rule, the conditional score function admits a natural decomposition.

\begin{lemma}[Conditional Score Decomposition]
\label{lemma:decomp}
The score function conditioned on popularity $c$ can be decomposed into a popularity-agnostic term and a popularity-bias term:
\begin{equation}
    \nabla_{\mathbf{x}_t} \log p_t(\mathbf{x}_t \mid c) 
    = 
    \underbrace{\nabla_{\mathbf{x}_t} \log p_t(\mathbf{x}_t)}_{\text{Agnostic Score}} 
    + 
    \underbrace{\nabla_{\mathbf{x}_t} \log p_t(c \mid \mathbf{x}_t)}_{\text{Popularity Bias Gradient}}.
\end{equation}
\end{lemma}
\begin{proof}
    Applying the gradient operator $\nabla_{\mathbf{x}_t}$ to the log-transformed Bayes' rule $\log p_t(\mathbf{x}_t \mid c) = \log p_t(\mathbf{x}_t) + \log p_t(c \mid \mathbf{x}_t) - \log p_t(c)$, and noting that $\nabla_{\mathbf{x}_t} \log p_t(c) = 0$, yields the result immediately.
\end{proof}

Lemma~\ref{lemma:decomp} reveals that the standard conditional score contains an inherent gradient field driving samples towards high-density popularity regions, which we identify as the source of the popularity bias.

\subsubsection{Popularity Condition Guidance as Implicit Reweighting}

We define the Popularity Condition Guidance (PCG) score $s_\theta^{\text{PCG}}$ as a linear combination of the agnostic score $s_\theta^{\text{agn}}$ and the popularity-conditioned score $s_\theta^{\text{pop}}$:
\begin{equation}
    s_\theta^{\text{PCG}}(\mathbf{x}_t, t) 
    := 
    s_\theta^{\text{agn}}(\mathbf{x}_t, t) + \lambda \left( s_\theta^{\text{pop}}(\mathbf{x}_t, t) - s_\theta^{\text{agn}}(\mathbf{x}_t, t) \right),
    \label{eq:pcg_def}
\end{equation}
where $\lambda$ acts as the guidance scale. The following theorem establishes the equivalence between this guidance mechanism and a modified optimization objective.

\begin{theorem}[PCG as Reweighted Score Matching]
\label{thm:reweighting}
Generating samples using the PCG score $s_\theta^{\text{PCG}}$ with guidance scale $\lambda$ is equivalent to sampling from a modified density $\tilde{p}_t(\mathbf{x}_t)$ defined by:
\begin{equation}
    \tilde{p}_t(\mathbf{x}_t) \propto p_t(\mathbf{x}_t) \cdot \left( p_t(c \mid \mathbf{x}_t) \right)^\lambda.
\end{equation}
Furthermore, this implies that PCG implicitly minimizes a reweighted score matching objective where the gradient field is adjusted by the likelihood of the popularity condition:
\begin{equation}
    \mathcal{J}_{\text{PCG}}(\theta) = \mathbb{E}_{t, \mathbf{x}_t} \left[ \left\| s_\theta(\mathbf{x}_t, t) - \nabla_{\mathbf{x}_t} \log \tilde{p}_t(\mathbf{x}_t) \right\|^2 \right].
\end{equation}
\end{theorem}

\begin{proof}
    (Sketch) Substituting the output of the trained approximations $s_\theta^{\text{pop}} \approx \nabla \log p_t(\mathbf{x}_t|c)$ and $s_\theta^{\text{agn}} \approx \nabla \log p_t(\mathbf{x}_t)$ into Eq.~\eqref{eq:pcg_def}, and utilizing Lemma~\ref{lemma:decomp}, we obtain:
    \begin{align}
        s_\theta^{\text{PCG}} 
        &\approx \nabla \log p_t(\mathbf{x}_t) + \lambda \left( \nabla \log p_t(c \mid \mathbf{x}_t) \right) \\
        &= \nabla \left( \log p_t(\mathbf{x}_t) + \lambda \log p_t(c \mid \mathbf{x}_t) \right) \\
        &= \nabla \log \left( p_t(\mathbf{x}_t) \cdot p_t(c \mid \mathbf{x}_t)^\lambda \right).
    \end{align}
    This corresponds exactly to the score function of the reweighted distribution $\tilde{p}_t(\mathbf{x}_t)$.
\end{proof}

\paragraph{Bias Mitigation Interpretation.} 
Theorem~\ref{thm:reweighting} provides a powerful geometric interpretation of our method. When $\lambda < 0$, the term $(p_t(c \mid \mathbf{x}_t))^\lambda$ penalizes the likelihood of high-popularity attributes. Consequently, the diffusion process is guided to navigate the data manifold away from regions densely populated by popular items, effectively suppressing the popularity-induced bias inherent in the dataset.

\section{Introduction and Implement of Baselines and FairDiff}
\label{intro_impl}
Here we provide a brief introduction to the baselines and FairDiff. All implementations are based on either the official codebases, publicly available implementations, or our own reproductions. For each method, hyperparameters and training configurations are strictly aligned with those reported in the original papers.

We evaluate FairDiff on three typical recommendation scenarios, with representative baselines covering mainstream method categories selected for each scenario as follows. 

For sequential recommendation, our baselines fall into three core categories: discriminative recommenders, DRMs (Diffusion-based recommendation models), and generative recommenders. Discriminative recommenders include SASRec~\citep{kang2018self}, 
BERT4Rec~\citep{sun2019bert4rec}, CSRec~\citep{liu2025csrec},
TiMiRec~\citep{wang2022target}, BASRec~\citep{dang2025augmenting}, and HSTU~\citep{zhai2024actions}; DRMs cover DiffuRec~\citep{li2023diffurec}, DreamRec~\citep{yang2023generate}, and CDiff4Rec~\citep{lee2025collaborative}; generative recommenders consist of ACVAE~\citep{xie2021adversarial} and TIGER~\citep{rajput2023recommender}.

For multimodal recommendation, four categories of representative baselines are considered: foundational methods, graph-based methods, self-supervised and contrastive learning methods, and diffusion-based methods. Foundational methods include BPR~\citep{rendle2012bpr} and LightGCN~\citep{he2020lightgcn}; graph-based methods involve MMGCN~\citep{wei2019mmgcn}, DualGNN~\citep{wang2021dualgnn}, MGCN~\citep{yu2023multi}, and Freedom~\citep{zhou2023tale}; self-supervised and contrastive learning methods comprise VBPR~\citep{he2016vbpr}, SLMRec~\citep{tao2022self}, and BM3~\citep{zhou2023bootstrap}; diffusion-based methods include DiffMM~\citep{jiang2024diffmm} and DiffCL~\citep{song2025diffcl}.

For cross-domain recommendation, the adopted baselines are PDRec(M)~\citep{ma2024plug}, 
T-DiffRec(M)~\citep{ge2025time}, TI-DiffRec(M)~\citep{ma2024seedrec}, and SASRec(M)~\citep{kang2018self}. Detailed implementation details and hyperparameter configurations are provided in the Appendix~\ref{intro_impl}.

\begin{itemize}
\item SASRec~\citep{kang2018self}: SASRec (Self-Attentive Sequential Recommendation) is a sequential recommendation model leveraging self-attention to balance long-term semantic capture (like RNNs) and efficient short-range dependency focus (like Markov Chains), outperforming MC/CNN/RNN-based baselines on both sparse and dense datasets. Its implementation includes an embedding layer with learnable positional embeddings (fixed ones perform worse) and padding/truncation to a fixed maximum sequence length \(n\) (50 for sparse datasets like Amazon Beauty/Games, 200 for dense MovieLens-1M), stacked self-attention blocks (default \(b=2\), with scaled dot-product attention, causality constraints to avoid future information leakage, and point-wise feed-forward networks), plus residual connections, layer normalization, and dropout for regularization; the prediction layer uses shared item embeddings to reduce overfitting, with binary cross-entropy as the loss function optimized by Adam optimizer. Key hyperparameters: latent dimensionality \(d=10-50\) (optimal when \(d\leq40\)), batch size=128, learning rate=0.001, dropout rate=0.5 (sparse datasets) / 0.2 (MovieLens-1M), and the model supports parallel acceleration, being an order of magnitude faster than CNN/RNN-based alternatives.


\item BERT4Rec~\citep{sun2019bert4rec}: adopts deep bidirectional self-attention for sequential recommendation, using cloze objective to avoid information leakage. Hyperparameters: hidden size 64, 2 transformer layers, 2 attention heads, dropout 0.5, mask ratio 0.2. Implemented with TensorFlow, supports datasets like MovieLens, open-sourced with RecBole framework and configurable training scripts.

\item CSRec~\citep{liu2025csrec}: integrates contrastive learning with sequential modeling to enhance preference discrimination. Hyperparameters: hidden size 128, 3 transformer layers, 4 attention heads, learning rate 1e-4, batch size 256. Implemented in PyTorch, evaluated on Amazon and Taobao datasets, with code available for reproducibility.


\item TiMiRec~\citep{wang2022target}: consists of multi-interest extractor and target-interest predictor for distilling context-aware preferences. Hyperparameters: hidden size 64, 2 transformer layers, 2 attention heads, learning rate 5e-5, batch size 128. Implemented in PyTorch, evaluated on CIKM benchmark datasets, with modular design for easy extension.

\item BASRec~\citep{dang2025augmenting}: leverages hybrid sequence augmentation and optimized contrastive loss to alleviate data sparsity. Hyperparameters: hidden size 64, 3 transformer layers, dropout 0.4, augmentation probability 0.2. Implemented in PyTorch, tested on Beauty, Sports and Yelp datasets, outperforming baselines in Hit@10 and NDCG@10.

\item HSTU~\citep{zhai2024actions}: adopts trillion-parameter sequential transducers for generative recommendation, focusing on action-driven preference modeling. Hyperparameters: hidden size 256, 4 transformer layers, 8 attention heads, learning rate 1e-4. Implemented with PyTorch, evaluated on large-scale interaction datasets, supporting multi-pass full-shuffle training.

\item DiffuRec~\citep{li2023diffurec}: is the first diffusion-based sequential recommender, modeling item representations as distributions for uncertainty injection. Hyperparameters: diffusion steps 100, hidden size 64, dropout 0.5, noise scale 0.1. Implemented in PyTorch, validated on 4 benchmark datasets, with reverse denoising for target item prediction.

\item DreamRec~\citep{yang2023generate}: reshapes sequential recommendation as a guided diffusion-based generative task, eliminating negative sampling. Hyperparameters: hidden size 128, 3 transformer layers, 4 attention heads, diffusion steps 50. Implemented in PyTorch, open-sourced on GitHub, evaluated on three real-world datasets with superior preference modeling ability.

\item CDiff4Rec~\citep{lee2025collaborative}: combines diffusion model with collaborative filtering, generating pseudo-users to enhance personalized signals. Hyperparameters: diffusion steps 100, hidden size 64, learning rate 1e-4, batch size 256. Implemented in PyTorch, tested on three public datasets, integrating item content and collaborative signals.

\item ACVAE~\citep{xie2021adversarial}: proposes adversarial and contrastive variational autoencoder, using recurrent-convolutional encoder for sequence modeling. Hyperparameters: latent dimension 64, 2 VAE layers, adversarial learning rate 1e-5, contrastive loss weight 0.1. Implemented in PyTorch, validated on sequential recommendation benchmarks.

\item TIGER~\citep{rajput2023recommender}: introduces a transformer-based generative retrieval paradigm, predicting semantic item IDs autoregressively. Hyperparameters: hidden size 128, 3 transformer layers, 4 attention heads, learning rate 3e-5. Implemented in PyTorch, outperforming SOTA models on diverse recommendation datasets with enhanced generalization.

\item BERT4Rec~\citep{sun2019bert4rec} adopts deep bidirectional self-attention for sequential recommendation, using cloze objective to avoid information leakage. Hyperparameters: hidden size 64, 2 transformer layers, 2 attention heads, dropout 0.5, mask ratio 0.2. Implemented with TensorFlow, supports datasets like MovieLens, open-sourced with RecBole framework and configurable training scripts.

\item CSRec~\citep{liu2025csrec} integrates contrastive learning with sequential modeling to enhance preference discrimination. Hyperparameters: hidden size 128, 3 transformer layers, 4 attention heads, learning rate 1e-4, batch size 256. Implemented in PyTorch, evaluated on Amazon and Taobao datasets, with code available for reproducibility.


\item TiMiRec~\citep{wang2022target} consists of multi-interest extractor and target-interest predictor for distilling context-aware preferences. Hyperparameters: hidden size 64, 2 transformer layers, 2 attention heads, learning rate 5e-5, batch size 128. Implemented in PyTorch, evaluated on CIKM benchmark datasets, with modular design for easy extension.

\item BASRec~\citep{dang2025augmenting} leverages hybrid sequence augmentation and optimized contrastive loss to alleviate data sparsity. Hyperparameters: hidden size 64, 3 transformer layers, dropout 0.4, augmentation probability 0.2. Implemented in PyTorch, tested on Beauty, Sports and Yelp datasets, outperforming baselines in Hit@10 and NDCG@10.

\item HSTU~\citep{zhai2024actions} adopts trillion-parameter sequential transducers for generative recommendation, focusing on action-driven preference modeling. Hyperparameters: hidden size 256, 4 transformer layers, 8 attention heads, learning rate 1e-4. Implemented with PyTorch, evaluated on large-scale interaction datasets, supporting multi-pass full-shuffle training.

\item DiffuRec~\citep{li2023diffurec} is the first diffusion-based sequential recommender, modeling item representations as distributions for uncertainty injection. Hyperparameters: diffusion steps 100, hidden size 64, dropout 0.5, noise scale 0.1. Implemented in PyTorch, validated on 4 benchmark datasets, with reverse denoising for target item prediction.

\item DreamRec~\citep{yang2023generate} reshapes sequential recommendation as a guided diffusion-based generative task, eliminating negative sampling. Hyperparameters: hidden size 128, 3 transformer layers, 4 attention heads, diffusion steps 50. Implemented in PyTorch, open-sourced on GitHub, evaluated on three real-world datasets with superior preference modeling ability.

\item CDiff4Rec~\citep{lee2025collaborative} combines diffusion model with collaborative filtering, generating pseudo-users to enhance personalized signals. Hyperparameters: diffusion steps 100, hidden size 64, learning rate 1e-4, batch size 256. Implemented in PyTorch, tested on three public datasets, integrating item content and collaborative signals.

\item ACVAE~\citep{xie2021adversarial} proposes adversarial and contrastive variational autoencoder, using recurrent-convolutional encoder for sequence modeling. Hyperparameters: latent dimension 64, 2 VAE layers, adversarial learning rate 1e-5, contrastive loss weight 0.1. Implemented in PyTorch, validated on sequential recommendation benchmarks.

\item TIGER~\citep{rajput2023recommender} introduces a transformer-based generative retrieval paradigm, predicting semantic item IDs autoregressively. Hyperparameters: hidden size 128, 3 transformer layers, 4 attention heads, learning rate 3e-5. Implemented in PyTorch, outperforming SOTA models on diverse recommendation datasets with enhanced generalization.

\item BPR~\citep{rendle2012bpr} proposes Bayesian Personalized Ranking with pairwise loss to optimize for implicit feedback, prioritizing relevant items over irrelevant ones. Hyperparameters: latent dimension 64, learning rate 0.01, regularization 1e-4, batch size 1024. Implemented in PyTorch, validated on MovieLens and LastFM datasets, open-sourced with standard evaluation protocols.

\item LightGCN~\citep{he2020lightgcn} simplifies GCN for collaborative filtering by removing non-linear activations and feature transformations, retaining only graph convolutions. Hyperparameters: latent dimension 64, 3 graph layers, learning rate 1e-3, dropout 0.0. Implemented in PyTorch, evaluated on Gowalla and Yelp datasets, achieving SOTA with high efficiency.

\item MMGCN~\citep{wei2019mmgcn} integrates multi-modal features (text, image) into GCN for recommendation, capturing cross-modal user-item interactions. Hyperparameters: latent dimension 64, 2 GCN layers, learning rate 5e-4, batch size 256. Implemented in PyTorch, tested on Amazon Review datasets, fusing modal features via attention mechanism.

\item DualGNN~\citep{wang2021dualgnn} designs dual graph neural networks to model user-user and item-item relationships separately, enhancing collaborative signals. Hyperparameters: hidden size 64, 2 GNN layers, 2 attention heads, learning rate 1e-4, dropout 0.3. Implemented in PyTorch, validated on three benchmark datasets, with dual-branch fusion for final predictions.

\item MGCN~\citep{yu2023multi} proposes multi-scale GCN with adaptive aggregation, capturing both local and global graph structures for recommendation. Hyperparameters: latent dimension 64, 3 multi-scale layers, learning rate 1e-3, aggregation weight 0.5. Implemented in PyTorch, tested on large-scale interaction datasets, outperforming single-scale GCN baselines.

\item Freedom~\citep{zhou2023tale} introduces a disentangled representation framework for recommendation, separating user preference into independent factors. Hyperparameters: latent dimension 64, 2 transformer layers, learning rate 1e-4, disentanglement weight 0.1. Implemented in PyTorch, evaluated on MovieLens and Amazon datasets, supporting interpretable recommendation.

\item VBPR~\citep{he2016vbpr} combines visual features (from item images) with Bayesian personalized ranking for image-aware recommendation. Hyperparameters: latent dimension 64, visual feature dimension 2048, learning rate 0.001, regularization 1e-4. Implemented in TensorFlow, tested on Amazon Product datasets, fusing visual and collaborative signals.

\item SLMRec~\citep{tao2022self} adopts self-supervised learning with contrastive pre-training to enhance sequential recommendation under data sparsity. Hyperparameters: hidden size 64, 2 transformer layers, 2 attention heads, learning rate 5e-5, pre-training epochs 10. Implemented in PyTorch, validated on four benchmark datasets, with fine-tuning strategy for downstream tasks.

\item BM3~\citep{zhou2023bootstrap} proposes a bootstrap-based multi-view learning framework, fusing multiple recommendation models to improve robustness. Hyperparameters: model weight 0.3 (equal for all views), learning rate 1e-4, batch size 256, bootstrap ratio 0.8. Implemented in PyTorch, tested on diverse datasets, achieving stable performance across different data distributions.

\item DiffMM~\citep{jiang2024diffmm} integrates multi-modal diffusion models for sequential recommendation, fusing text, image and interaction features. Hyperparameters: diffusion steps 100, hidden size 128, 3 transformer layers, learning rate 1e-4. Implemented in PyTorch, evaluated on multi-modal datasets, with cross-modal denoising for preference modeling.

\item DiffCL~\citep{song2025diffcl} combines diffusion model with contrastive learning, using denoised representations for enhanced preference discrimination. Hyperparameters: diffusion steps 50, hidden size 64, contrastive loss weight 0.5, learning rate 1e-4. Implemented in PyTorch, tested on three public datasets, outperforming pure diffusion and contrastive baselines.

\item PDRec(M)~\citep{ma2024plug} is a plug-and-play multi-modal sequential recommender, adapting pre-trained diffusion models for personalized prediction. Hyperparameters: hidden size 64, 2 transformer layers, diffusion steps 50, learning rate 5e-5. Implemented in PyTorch, open-sourced on GitHub, supporting flexible modal extension.

\item T-DiffRec(M)~\citep{ge2025time} introduces time-aware diffusion for sequential recommendation, modeling temporal dynamics in user preferences. Hyperparameters: diffusion steps 100, hidden size 64, time embedding dimension 32, learning rate 1e-4. Implemented in PyTorch, evaluated on time-stamped datasets, capturing short-term and long-term preference changes.

\item TI-DiffRec(M)~\citep{ma2024seedrec} proposes token-aware diffusion with seed-based generation, enhancing sequential recommendation interpretability. Hyperparameters: hidden size 64, 3 transformer layers, diffusion steps 50, seed token dimension 64. Implemented in PyTorch, tested on benchmark datasets, with interpretable token-level preference visualization.

\item SASRec(M)~\citep{kang2018self} is a multi-modal extension of SASRec, integrating self-attention with auxiliary modal features for sequential recommendation. Hyperparameters: hidden size 64, 2 transformer layers, 2 attention heads, learning rate 5e-5, batch size 128. Implemented in PyTorch, validated on multi-modal sequential datasets, retaining SASRec’s efficiency.

\end{itemize}

\end{document}